\documentclass[twoside]{article}

\usepackage[accepted]{aistats2021}
\usepackage[utf8]{inputenc}
\usepackage{url}                                        
\usepackage{booktabs}                                   
\usepackage{amsfonts}                                   
\usepackage{nicefrac}                                   
\usepackage{microtype}                                  
\usepackage{bookmark}                                   
\usepackage{amsthm}                                     
\usepackage{amsmath}
\usepackage{nccmath}
\usepackage{tabulary}                                   
\usepackage{amssymb}                                    
\usepackage{enumitem}                                   

\usepackage{enumitem}                                   
\usepackage{bbm}                                        
\usepackage[backend=biber, natbib=true, style=authoryear, maxnames=50, maxcitenames=1]{biblatex}  
\usepackage{xargs}                                      
\usepackage[pdftex,dvipsnames]{xcolor}                  
\usepackage{tcolorbox}                                  

\usepackage[ruled,vlined]{algorithm2e}                  
\usepackage{algorithmic}                                

\usepackage[colorinlistoftodos,prependcaption,textsize=tiny]{todonotes}
\newcommandx{\unsure}[2][1=]{\todo[linecolor=red,backgroundcolor=red!25,bordercolor=red,#1]{#2}}
\newcommandx{\change}[2][1=]{\todo[linecolor=blue,backgroundcolor=blue!25,bordercolor=blue,#1]{#2}}
\newcommandx{\info}[2][1=]{\todo[linecolor=OliveGreen,backgroundcolor=OliveGreen!25,bordercolor=OliveGreen,#1]{#2}}
\newcommandx{\improvement}[2][1=]{\todo[linecolor=Plum,backgroundcolor=Plum!25,bordercolor=Plum,#1]{#2}}
\newcommandx{\thiswillnotshow}[2][1=]{\todo[disable,#1]{#2}}

\newcommandx{\inlinetodo}[2][1=]{\todo[inline, linecolor=Plum,backgroundcolor=Plum!25,bordercolor=Plum,#1]{#2}}

\usepackage{pgf,tikz,pgfplots}
\pgfplotsset{compat=1.15}
\usepackage{mathrsfs}
\usetikzlibrary{arrows}

\definecolor{uuuuuu}{rgb}{0.26666666666666666,0.26666666666666666,0.26666666666666666}
\definecolor{ffwwzz}{rgb}{1,0.4,0.6}
\definecolor{zzttff}{rgb}{0.6,0.2,1}
\usepackage{layouts}
\newtheorem{theorem}{Theorem}                
\newtheorem{proposition}{Proposition}        
\newtheorem{lemma}{Lemma}                     

\theoremstyle{definition}                               
\newtheorem{definition}{Definition}[section]

\newenvironment{claim}[1]{\par\noindent\underline{Claim:}\space#1}{}
\newenvironment{claimproof}[1]{\par\noindent\underline{Proof:}\space#1}{\hfill $\blacksquare$}

\newcommand{\prob}{\mathcal{P}}                 
\newcommand{\mixingSpace}{\prob^{2}_K}

\newcommand{\mixset}{\mathfrak{L}}
\newcommand{\trueMsr}{\Gamma}
\newcommand{\trueMixMsr}{\Lambda}
\newcommand{\trueCmpMsr}{\gamma}
\newcommand{\trueCmpwgt}{\lambda}

\newcommand{\compKDEf}{\widetilde{f}}
\newcommand{\compKDEGamma}{\psi}

\newcommand{\estMsr}{\widehat{\Gamma}}
\newcommand{\estMixMsr}{\widehat{\Lambda}}
\newcommand{\estCmpMsr}{\widehat{\gamma}}
\newcommand{\estCmpwgt}{\widehat{\lambda}}
\newcommand{\estCmpdensity}{\hat{f}}

\newcommand{\domain}{\mathbb{R}^d}
\newcommand{\rkhs}{\mathcal{H}}
\newcommand{\norm}[1]{\left\lVert#1\right\rVert}
\newcommand{\RKHSnorm}[1]{\lVert#1\rVert_{\mathcal{H}_{G_{\zeta}}}}
\newcommand{\inner}[1]{\langle #1 \rangle}

\newcommand{\nSample}{X = \left \{ x_1, x_2, \cdots x_n \right \}}

\newcommand{\sample}{X}
\newcommand{\mySum}{\sum \limits}
\newcommand{\myBracs}[1]{\left ( #1 \right )}
\newcommand{\myCurls}[1]{\left \{ #1 \right \}}
\newcommand{\kde}{\widehat{f}}

\newcommand{\BayesPart}{\sigma_{Bayes}}
\newcommand{\truePart}{\sigma^*}

\newcommand{\R}{\mathbb{R}}
\newcommand{\N}{\mathbb{N}}

\DeclareMathOperator*{\argmax}{arg\,max}

\DeclareMathOperator*{\argsup}{arg\,sup}
\DeclareMathOperator*{\argmin}{arg\,min}

\newcommand{\actr}{${\textrm{CTR}}$}
\newcommand{\affk}{${\textrm{FFK}}$}
\newcommand{\alnk}{${\textrm{LNK}}$}
\newcommand{\akmn}{${\textrm{KMN}}$}
\newcommand{\akde}{$\mathcal{A}_{\textrm{KDE}}$}

\newcommand{\ectr}{$\mathcal{E}_{\textrm{CTR}}$}
\newcommand{\effk}{$\mathcal{E}_{\textrm{FFK}}$}
\newcommand{\elnk}{$\mathcal{E}_{\textrm{LNK}}$}

\newcommand{\abs}[1]{\vert #1 \vert}
\newcommand{\mynorm}[1]{\vert \vert #1 \vert \vert}

\let\emptyset\varnothing
\usepackage{hyperref}       
\begin{document}

%

%
\runningauthor{Vankadara, Bordt, von Luxburg, Ghoshdastidar}

\twocolumn[

\aistatstitle{Recovery Guarantees for Kernel-based Clustering under Non-parametric Mixture Models}

\aistatsauthor{ 
Leena C. Vankadara{\normalfont\textsuperscript{1}}~~~Sebastian Bordt{\normalfont\textsuperscript{1,2}}~~~Ulrike von Luxburg{\normalfont\textsuperscript{1,2}}
~~~Debarghya Ghoshdastidar{\normalfont\textsuperscript{3}}}

\aistatsaddress{
University of T{\"u}bingen{\normalfont\textsuperscript{1}}
\And 
Max Planck Institute{}
\\for Intelligent Systems, T{\"u}bingen\normalfont\textsuperscript{2}
\And 
Technical University of \\ Munich{\normalfont\textsuperscript{3}}} ]

\begin{abstract}
    Despite the ubiquity of kernel-based clustering, surprisingly few statistical guarantees exist beyond settings that consider strong structural assumptions on the data generation process. In this work, we take a step towards bridging this gap by studying the statistical performance of kernel-based clustering algorithms under non-parametric mixture models. We provide necessary and sufficient separability conditions under which these algorithms can consistently recover the underlying \textit{true clustering}. Our analysis provides guarantees for kernel clustering approaches without structural assumptions on the form of the component distributions. Additionally, we establish a key equivalence between kernel-based data-clustering and kernel density-based clustering. This enables us to provide consistency guarantees for kernel-based estimators of non-parametric mixture models. Along with theoretical implications, this connection could have practical implications, including in the systematic choice of the bandwidth of the Gaussian kernel in the context of clustering.
\end{abstract}

\section{INTRODUCTION}
Clustering refers to the unsupervised task of partitioning a given data sample or the input space into \textit{meaningful} regions. Kernel clustering approaches such as kernel k-means \citep{dhillon2004kernel} and kernel spectral clustering \citep{ng2002spectral} are widely adopted by practitioners, particularly for partitioning non-spherical complex cluster structures. Beyond their good practical behavior, kernel methods are appealing due to their amenability to theoretical analysis. However, as an anomaly, kernel clustering has been elusive to theoretical analysis, in particular, under general non-parametric assumptions on the data generation process. One of the principle sources for this gap between theory and practice had been the lack of a universally accepted characterization of the quality of a clustering. One popular notion of the goodness of clustering is defined as the one that consistently partitions the data space. Consistency is, however, only a necessary condition for clustering algorithms. It simply checks if an algorithm asymptotically converges to a limiting partition. The optimality of this limiting partition is not studied under consistency. As an example, spectral clustering has been shown to be consistent \citep{vonluxburg20008consistency} for any similarity function $k$. However, if one uses a similarity function based on an uninformative kernel such as the identity kernel, then the obtained limiting partition is clearly not guaranteed to be a desirable one. Density based clustering \citep{hartigan1975clustering, hartigan1981consistency, rinaldo2010generalized} is another popular line of work with theoretical backing, where clusters are defined as connected components of high-density regions, referred to as density level sets. The imprecise notion of a high-density region is overcome using the so called cluster-tree approach \citep{chaudhuri2014consistent, sriperumbudur2012consistency}, where a continuum of all level sets is simultaneously considered.

Another systematic approach to overcome the ambiguity concerning the quality of clustering lies in the so called \textit{model-based clustering}, which assumes that the data is generated from a mixture distribution and the goal is to partition the data in congruity with the components that generate the data. However, theoretical analysis of  kernel clustering methods have been confined to settings with parametric distributions \citep{yan2016robustness, couillet2016kernel,vankadara2019optimality}. Parametric assumptions such as the Gaussian mixture setting, where the components are assumed to be normally distributed, are extremely restrictive since the data generated under such assumptions are far from a typical dataset for which kernel clustering algorithms are applicable. In contrast, non-parametric assumptions on the data-generation process can be considerably less restrictive, but kernel clustering algorithms have been elusive to theoretical analysis under such assumptions. A primary hurdle in the analysis of clustering approaches under non-parametric assumptions is due to the issue of identifiability of non-parametric mixture models, that is, non-parametric models may be ambiguously defined. There is limited previous work that presents an analysis of kernel-based clustering algorithms under non-parametric mixture models. \citet{schiebinger2015geometry} provide recovery guarantees for spectral clustering of non-parametric mixtures by analyzing the spectral properties of the Laplacian operator under the assumption that the overlap between the components is small relative to a notion of ``indivisibility'' of the components. The analysis provided in \citet{schiebinger2015geometry} is restricted to that of spectral clustering and considerably different from the analysis in this paper.

\subsection{Contributions} 

\textbf{Non-parameteric kernel clustering.}
 We provide non-parametric conditions for consistency of certain kernel-based clustering algorithms. To the best of our knowledge, these are among the first theoretical guarantees to kernel-based clustering methods without assumptions on the form of the component distributions.
\begin{enumerate}[topsep=0pt]

\item 
We provide an {\bf impossibility result for kernel k-means}: there exists a mixture distribution with arbitrarily large separation between the components such that for finite samples from this distribution kernel k-means fails to recover the underlying clustering.
    
\item We establish {\bf sufficient separability conditions} under which kernel-based algorithms such as k-center, farthest-first k-means (FFk-means++), or kernel linkage algorithms can consistently recover the true partition, given finite samples from a mixture distribution.
    
\item We establish {\bf necessary conditions for consistency} of the kernel FFk-means++ and kernel linkage algorithms and show that these separability conditions are optimal, that is, the sufficient conditions match the necessary conditions.
\end{enumerate}

\textbf{Kernel-based data clustering as distribution clustering.} We establish a key equivalence between kernel-based data clustering and kernel-based density clustering. In particular:
\begin{enumerate}[topsep=0pt]
\setcounter{enumi}{3}
   
\item We show that Gaussian kernel-based data clustering is equivalent to density clustering, where, each data point is first represented by a Gaussian probability density function and the densities are then clustered using the maximum mean discrepancy metric (with respect to a Gaussian kernel).

\item In addition to theoretical implications, this connection could also have practical implications in matters such as choosing the bandwidth of the Gaussian kernel for clustering which has not been systematically studied in literature so far. Our analysis reveals that the bandwidth of the kernel used for clustering needs to decrease with $n$ but, perhaps surprisingly, asymptotically remain non-zero.
\end{enumerate}

\textbf{Non-parametric estimation of mixture models.}
Due to this relationship between kernel data clustering and distribution clustering, any standard Gaussian kernel clustering algorithm can be used to define an estimation procedure of the mixture model. 
Therefore, in addition to our primary contributions to kernel clustering, we also make contributions related to non-parametric \textit{estimation} of mixture models.
\begin{enumerate}[topsep=0pt]
\setcounter{enumi}{5}
    \item We provide conditions under which the \textit{estimation procedures} corresponding to the kernel-based clustering algorithms can consistently estimate the true mixture model. 
   
\end{enumerate}
\section{FORMAL SETTING AND BACKGROUND}
\label{sec:prelims}

Consider the Euclidean space $\R^d$ of dimension $d$ as the input domain. Let $\prob$ denote the space of all Borel probability measures on $\R^d$ that are absolutely continuous with respect to the Lebesgue measure. In our analysis, we use the framework of mixing measures to define mixture distributions. This is fairly standard in the analysis of non-parametric mixture models \citep{aragam2018identifiability, holzmann2006identifiability, kimeldorf1970correspondence, nguyen2013convergence, teicher1963identifiability} primarily due to the following reasons:
\begin{itemize}[topsep=0pt]
    \item Arbitrary mixture distributions are not identifiable. Mixing measures allow for the specification of \textit{true components}. Section \ref{def:identifiability} provides a thorough discussion on identifiability of mixture models.
    \item In non-parametric clustering, one typically does not make any assumptions on the form of the component distributions. An elegant way to accomplish this is to allow arbitrary component distributions from $\prob$ and impose restrictions on the set of admissible mixing measures. 
\end{itemize}
Following the notation of \citet{aragam2018identifiability}, we denote the space of all probability distributions (mixing measures) over $\prob$ supported on a finite ($K$) number of elements in $\prob$ by $\mixingSpace.$ Formally, 
$$\mixingSpace = \Bigg \{ {\sum \limits_{k = 1}^K \trueCmpwgt_k \delta_{\trueCmpMsr_k} \textbf{:} \; \trueCmpwgt_k \in \R^{+}, \; \trueCmpMsr_k \in \prob, \; \sum \limits_{k =1 }^K \trueCmpwgt_k = 1 } \Bigg \},$$
where $\delta_{\gamma}$ denotes the point mass concentrated at $\gamma \in \prob$ and $[K] $ denotes the set $\myCurls{1, 2, \cdots K}$ for any $K \in \N.$ Furthermore, assume that the coefficients ($\trueCmpwgt_k$) of the component measures ($\gamma_k$) are bounded away from $0$. Define $m:\mixingSpace \rightarrow \prob$ to be the mapping that uniquely associates a \textbf{mixing measure} to a \textbf{mixture distribution}, that is,
\begin{equation*}
    \forall \; \trueMixMsr \in \prob_K^2 : \trueMixMsr  = \sum \limits_{k = 1}^{K} \trueCmpwgt_k \delta_{\trueCmpMsr_k} \longrightarrow m(\trueMixMsr) = \sum \limits_{k = 1}^{K} \trueCmpwgt_k \trueCmpMsr_k.
\end{equation*}
The support of a mixing measure $\trueMixMsr$ specifies the true components of the corresponding mixture distribution, $\trueMsr = m(\trueMixMsr).$

We now describe the \textbf{problem setup}. Let $\trueMixMsr = \sum_{k \in [K]} \trueCmpwgt_k \delta_{\trueCmpMsr_k}$ be a mixing measure in $ \mixingSpace$. Consider a finite sample $\nSample$ drawn independently and identically (i.i.d) according to some $\trueMsr = m(\trueMixMsr) = \sum \limits_{k=1}^{K} \trueCmpwgt_k \trueCmpMsr_k$. We denote this by $X \sim \trueMsr^n$. The component measures $\trueCmpMsr_k$ are absolutely continuous with respect to the Lebesgue measure, and therefore admit density functions. We use $f_k$ to denote the density function corresponding to the component measure $\trueCmpMsr_k$ and $f = \sum \limits_{k = 1}^K \trueCmpwgt_k f_k$ to denote the density function corresponding to $\trueMsr.$ Given any density function $h$, we use the term ``probability distribution corresponding to $h$'' to denote the measure $\psi$ which is defined as $\psi_i(A) = \int_{A} h(x) dx$, for any Borel set $A \subseteq \R^d.$

 For any sample $\nSample$, we use a map $\sigma:[n] \rightarrow [K]$ to represent a $K-$partition of $X$ and $c_k(\sigma) = \myCurls{x_i \in \sample: \sigma(i) = k}$ to denote the $k^{\textrm{th}}$ cluster according to $\sigma$ for all $k \in [K]$. When it is clear from context, we drop the dependence on $\sigma$ and simply use $c_k$ to denote $c_k(\sigma)$. Given any $X \sim \trueMsr^n$, the ``planted partition'' and the ``Bayes partition'' are of particular interest.
 
 \textbf{Planted partition.} Observe that, drawing a sample $\nSample$ according to a mixing measure $\trueMixMsr = \sum_{k \in [K]} \trueCmpwgt_k \delta_{\trueCmpMsr_k}$ is equivalent to the following procedure. For each $i \in [n]$, 
 \begin{enumerate}[topsep=0pt]
     \item sample index $k\in [K]$ using the weights $\trueCmpwgt_1,\ldots,\trueCmpwgt_k$, 
     \item generate a sample $x_i$ from $\trueCmpMsr_k$. 
 \end{enumerate}
  We refer to the partition induced by this process as the planted partition and use $\truePart_{\sample}$ or $\truePart_{n}$ to denote it. 
 
\textbf{Bayes partition.} We refer to the mapping $ b^* : \sample \rightarrow [K]$ as the Bayes partition function, given by 
\begin{equation*}
    \BayesPart(x) = \argmax_{k} \trueCmpwgt_k f_k(x).
\end{equation*}
We use $\BayesPart^{X}$ to denote the Bayes partition with respect to a sample $\sample \sim m(\trueMixMsr)^n$ which is defined as the Bayes partition function restricted to $\sample.$

\textbf{Remark.}
In this work, any reference to a sample should be understood as drawn i.i.d according to a mixture distribution $\trueMsr.$

We now describe the main objective of this work: \textbf{clustering of non-parametric mixture models.} 
\begin{center}
    \textit{\textbf{Non-parametric clustering.} Given a finite sample $\nSample$ drawn i.i.d according to $\trueMsr^n$, the central objective of non-parametric, model-based clustering is to recover the planted partition up to a permutation over the labels, $[K]$.}
\end{center}
  Alternatively, one could also be interested in the consistent estimation of the Bayes partition \citep{aragam2018identifiability}. We present our results with respect to the former notion and they can easily be extended to the latter by means of a simple modification of the algorithms. We discuss this in more detail in Section \ref{sec:estimation_and_Bayes_partition}. The primary objective of this paper is to understand the performance of kernel clustering algorithms under the framework of non-parametric clustering. A brief background on kernels is thus warranted for further discourse on our analysis.

\textbf{Background on kernels.} Every symmetric positive definite (p.d) kernel function $g: \domain \times \domain \rightarrow \R$ is associated with a feature map $\phi: \domain \rightarrow \rkhs_g$, where $\rkhs_g$ is a Hilbert space with the inner product $\langle \cdot, \cdot \rangle_{\rkhs_g}$ such that $\inner{\phi(x), \phi(y)}_{\rkhs_g} = g(x,y), \textrm{ } \forall x,y \in \domain$. $\rkhs_g$ is a \textbf{reproducing kernel Hilbert space} (RKHS) if the mapping $f\mapsto f(x)$ is continuous for every $x\in\R^d$, where $f \in \rkhs_g$.  
%
%
The Hilbert space $\rkhs_g$ corresponding to a kernel $g$ is of independent interest while dealing with probability measures since it admits feature representations referred to as the \textbf{kernel mean embeddings}. For any probability measure $P \in \prob$, the kernel mean embedding with respect to kernel $g$ is defined as $\mu_P(\cdot) = \int_{x \in \domain} g(x, \cdot) dP$, which is an element of $\rkhs_g$. The RKHS norm $\RKHSnorm{\cdot}$ associated with $\rkhs_g$ can be used to define a (semi-)metric between the probability measures. Formally, the maximum mean discrepancy (MMD) between two probability measures $P, Q \in \prob$ with respect to the kernel $g$ is given by $\rho(P, Q) = \RKHSnorm{\mu_P - \mu_Q}$. If $g$ is a characteristic kernel, such as the Gaussian kernel, then $\rho$ is a metric on the space of probability measures $\prob$ \citep{fukumizu2008kernel, sriperumbudur2010hilbert}. In our analysis, we consider the space $\prob$ metrized by the MMD corresponding to a Gaussian kernel function, $g_{\zeta}: \domain \times \domain \rightarrow \R$, where $g_{\zeta}(x, y) = \exp\myBracs{{-\frac{\norm{x - y}^2}{\zeta}}} \; \forall x, y \in \domain$ with bandwidth $ \zeta > 0.$ The MMD metric enjoys several valuable properties, from both a theoretical and practical point of view \citep{gretton2012kernel, muandet2016kernel}. 
\textbf{Kernel density estimation} is a popular non-parametric approach for density estimation. Given any $\nSample \sim \trueMsr^n$, the kernel density estimate (KDE) of the density function $f$, with respect to Gaussian kernel $g_\beta$ with bandwidth $\beta > 0$, is given by
\begin{equation}
\label{eqn:kde}
    \kde(x) = \frac{1}{n} \mySum_{i = 1}^{n} \widetilde{f}_{i} \myBracs{x}; \quad \widetilde{f}_{i}(x) = \frac{\exp\myBracs{{-\frac{\norm{x - x_i}^2}{2 \beta^2}}}}{(2 \pi \beta^2)^{d/2}} .
\end{equation}
Let $\estMsr,\compKDEGamma_i\in\prob$ be the probability distributions corresponding to $f, \widetilde{f}_i$ respectively.
Under the following conditions on the \textbf{bandwidth parameter $\beta$},
\begin{equation}
    \label{eq:bandwidth}
    \beta \rightarrow 0, \quad \frac{n \beta^d}{\log n} \rightarrow \infty \; \textrm{as} \; n \rightarrow \infty,
\end{equation}
the kernel density estimate $\hat{f}_n$ converges to the true density $f$ in the $l_{\infty}$ norm \citep{gine2002rates, einmahl2005uniform}. 
%
%
\section{RECOVERY GUARANTEES FOR KERNEL-BASED DATA CLUSTERING}
\textbf{Identifiability.} A key theoretical question concerning both estimation and clustering under non-parametric mixture models is that of identifiability, that is, any mixture distribution can be decomposed in infinitely many ways into component distributions \citep{ teicher1963identifiability, holzmann2006identifiability, vandermeulen2015identifiability, miao2016identifiability, aragam2018identifiability}. Therefore, non-parametric clustering and estimation of mixture models are ill-defined, even if the number of components $K$ is assumed to be known. The framework of mixing measures as discussed earlier allow for the specification of the ``true components'' and the ``true planted/Bayes partitions''. For any set of mixing measures $\mixset \subseteq \mixingSpace$, let $m(\mixset)$ denote the set of mixture distributions corresponding to $\mixset$. Clearly, the mapping $\mixset \mapsto m(\mixset)$ is not injective on the whole space $\mixset = \prob^2$ due to general non-identifiability. This motivates the following definition. 
\begin{definition}[\textbf{Identifiablility}]
\label{def:identifiability}
A subset $\mathfrak{L} \subseteq \mixingSpace$ is called identifiable if the map $\mathfrak{L} \mapsto m(\mathfrak{L})$ is injective.
\end{definition}

The most common approach to deal with identifiability is to make restrictive parametric assumptions on the form of the component distributions, for example, Gaussianity, which renders the mixture model identifiable \citep{bruni1985identifiability, teicher1963identifiability}. Recent work by \citet{aragam2018identifiability} uses regularity and separability criteria to achieve identifiability.  Our analysis, inspired by \citet{aragam2018identifiability}, also uses separability criterion to deal with identifiability. However, our analysis differs from theirs on several fronts since we do not impose any regularity conditions on the mixing measures and also consider a statistical approach to identifiability. Moreover, the focus of their paper (identifiability of non-parametric mixture models) is very different from ours, which is providing recovery guarantees for kernel-based clustering approaches.

Any non-parametric analysis of model-based clustering (or estimation) is typically preceded by an identifiability analysis for the mixture models. We do not explicitly study identifiability, that is, identifying a set $\mixset \in \mixingSpace$ for which only one mixing measure can generate a mixture distribution. Instead, given finite samples from the mixture distribution, we provide conditions under which a particular algorithm (is biased toward and hence) recovers the true mixing measure/partition. In our analysis of kernel-based clustering algorithms, we show that under appropriate separability conditions, certain algorithms can consistently recover the planted partition. Specifically, we present and analyze the asymptotic behavior of four different kernel-based clustering algorithms. 

\textbf{Algorithms.}
We present a brief description of the algorithms here for completeness and include detailed descriptions in the supplementary. Consider a finite sample $\nSample \sim \trueMsr^n.$
\begin{itemize}[topsep=0pt]
    \item \textbf{k-means (\akmn).} The objective is to find a partition $\widehat{\sigma}:[n] \rightarrow [K]$ such that the sum of squared within cluster distances on $X$ is minimized. We consider the optimal solution to the NP-Hard, k-means problem in our analysis.
    
    \item \textbf {FFk-means$++$ (\affk).} This algorithm is a variant of k-means++ where the initial centers are chosen in a deterministic, farthest-first order. 
    
    \item \textbf{k-center (\actr).} The objective seeks to obtain a k-partition of $X$ such that the maximal radius of the clusters is minimized. The optimal solution to the NP-Hard k-center problem is analyzed.
    
   \item  \textbf{Agglomerative linkage (\alnk).} Given a similarity function (single, average or complete linkage), these algorithms generate a dendrogram establishing a hierarchy of clusters of the data in a bottom up approach, starting out with each point as its own cluster and progressively combining them into larger clusters until there is a single cluster that contains the entire data.
\end{itemize}
Given a positive definite kernel $g:\R^d \times \R^d \rightarrow \R$, the kernelized versions of these algorithms are defined by replacing the Euclidean inner product by the inner product $\inner{\cdot, \cdot}_{g}$ induced by $g$ on the input space $\R^d$, which is given by
\begin{equation*}
    \inner{x_i, x_j}_g =  g(x_i, x_j).
\end{equation*}
In this paper, we provide necessary and sufficient separability conditions for the kernel-based clustering algorithms \akmn, \affk, \actr, and \alnk.
\begin{figure}
    \centering
    \includegraphics[width=0.4\textwidth]{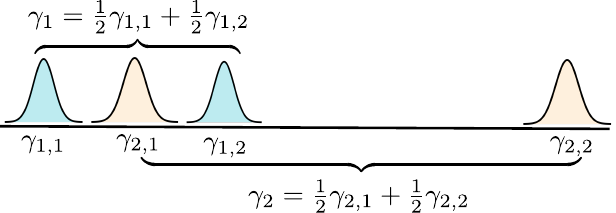}
    \caption{Example to show that simple separation conditions do not suffice to overcome identifiability. As the distribution $\gamma_{2,2}$ moves arbitrarily far from the remaining distributions, the distance between $\gamma_1$ and $\gamma_2$ also increases arbitrarily. However, without additional assumptions, no clustering algorithm can recover the \textbf{desirable clusters} as defined by the \textbf{true components} $\gamma_1$ and $\gamma_2.$}
    \label{fig:identifiability_example}
\end{figure}
\begin{figure*}
    \centering
    \includegraphics[width=\textwidth]{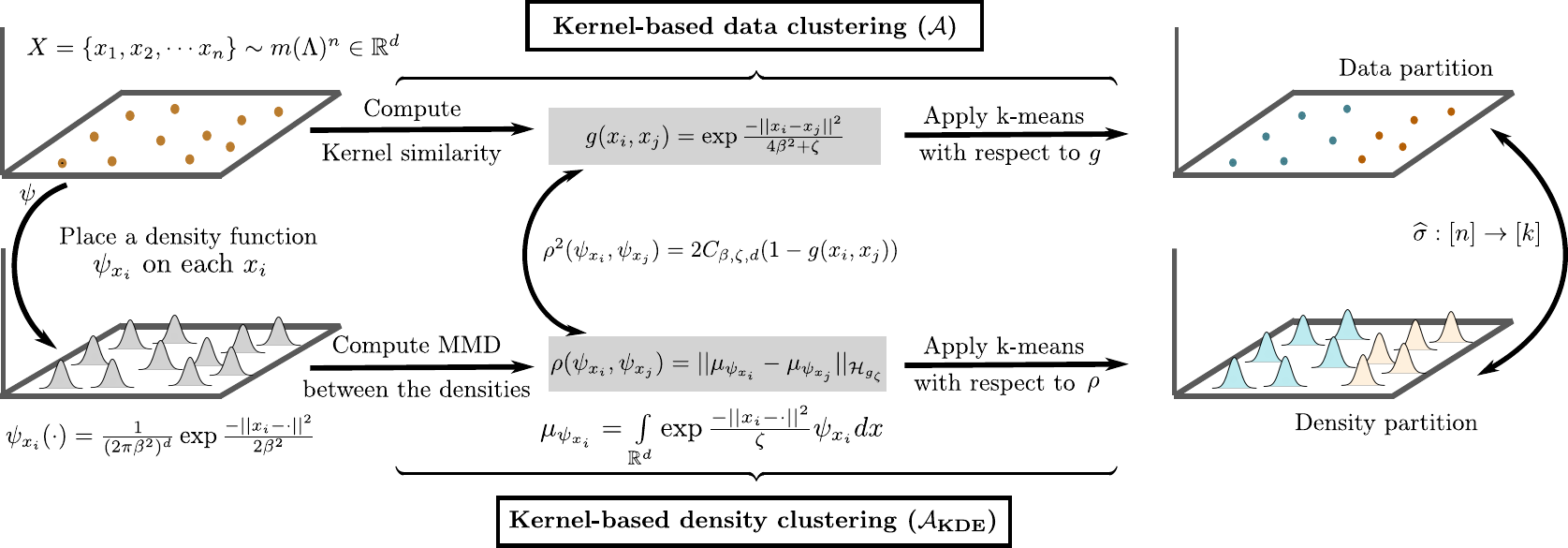}
    \caption{Illustration of the equivalence between kernel-based data clustering and distribution clustering. For Gaussian kernel clustering algorithm $\mathcal{A}$ using a bandwidth parameter $\eta > 0$, decompose $\eta$ to obtain \textit{any} $\beta >0$ and $\zeta >0$ satisfying $ 4 \beta^2 +\zeta = \eta.$ Then $\mathcal{A}$ can equivalently be reformulated as a kernel-based density clustering procedure as shown in the figure.}
    \label{fig:kde_equivalence}
\end{figure*}

\textbf{Main results.}
 For a finite sample $\nSample \sim \trueMsr^n$, recall that $\psi_{x_i}$ refers to the probability distribution corresponding to $\compKDEf_i$ as defined in (\ref{eqn:kde}) with bandwidth parameter $\beta > 0.$ Given a partition $\sigma:[n] \rightarrow [K]$ of probability distributions $\myCurls{\compKDEGamma_i}_{i=1}^n$, we use $\estCmpMsr_{k, \sigma}$ to denote the mean of the $k^{\textrm{th}}$ cluster according to $\sigma$, that is, $$ \estCmpMsr_{k, \sigma} = \frac{1}{\abs{c_k(\sigma)}} \sum \limits_{x_i \in c_k(\sigma)} \psi_{x_i}.$$  Let $\rho$ denote the MMD corresponding to the Gaussian kernel $g_{\zeta}$ with respect to a bandwidth parameter $\zeta > 0$ and let $g$ denote the Gaussian kernel function with the bandwidth parameter $(4 \beta^2 + \zeta).$ For readability, when it is clear from context, we ignore the dependence on the partition function, $\sigma$ in the notation. We now present one of our key results which establishes the impossibility of cluster recovery for kernel k-means. The result states that there is always a mixing measure with arbitrarily large MMD separation between the component distributions for which, given finite samples from this mixture, kernel k-means fails to recover the planted clustering.
 \begin{theorem}[\textbf{Impossibility of clustering recovery by \akmn}]
\label{thm:impossibility}
Fix $\zeta > 0.$ Let $\beta$ be any sequence of bandwidth parameters and let $g$ be the Gaussian kernel with bandwidth parameter $4 \beta^2 +\zeta$. For all $C > 0$, there exists a mixing measure $\trueMixMsr \in \mathcal{P}^2_2$ such that
\begin{equation}
\label{eq:suff_conds_ctr_supp}
    \rho(\trueCmpMsr_1, \trueCmpMsr_{2}) > C\sup \limits_{\substack{x \in \sample_n}} \rho(\compKDEGamma_x, \estCmpMsr_{\sigma^*(x), \sigma^*})
\end{equation}
holds within all finite samples and yet \akmn \textrm{ }with kernel $g$ w.h.p. fails to recover the planted partition $\sigma^*$.
\end{theorem}
Even though kernel k-means fails to provably recover the planted partition for arbitrarily large separation between the components, there is a sufficient separation between the components beyond which kernel-based k-center, FFk-means++, and hierarchical linkage algorithms can provably and consistently recover the planted partition. 

\begin{theorem}[\textbf{Sufficient conditions for consistency of \actr, \affk, and \alnk}]
\label{thm:consistency_sufficient}
Fix $\zeta > 0.$ Let $\beta$ be any sequence of bandwidth parameters satisfying (\ref{eq:bandwidth}) and let $g$ be the Gaussian kernel with bandwidth parameter $4 \beta^2 +\zeta$. For any $\trueMixMsr \in \prob^2_{K}$, if there exists $\epsilon > 0$ such that
\begin{multline}
\label{eq:suff_conds}
     \mathbb{P}_{X_n}\left(\inf \limits_{k \neq k'} \rho(\trueCmpMsr_k, \trueCmpMsr_{k'}) > 4\sup \limits_{\substack{x \in \sample_n}} \rho(\compKDEGamma_x, \estCmpMsr_{\sigma^*(x), \sigma^*})  + \epsilon \right)\\\stackrel{n \rightarrow \infty}{\longrightarrow} 1,
\end{multline}
 then the algorithms $\mathcal{A}_{\textrm{CTR}}$, $\mathcal{A}_{\textrm{FFK}}$, and $\mathcal{A}_{\textrm{LNK}}$ with kernel $g$ can w.h.p. recover the planted partition $\sigma^\star$.
\end{theorem}
The result states that, for recovery, the distance between any two component distributions in MMD ($\rho$) needs to be larger than about twice the maximal within cluster distance in the feature space: the RKHS ($\rkhs_g$) corresponding to the kernel $g$, for clustering defined by the planted partition. The conditions provided here might appear to be weak, but perhaps more consequentially, in Theorem \ref{thm:necessary} we show that under no additional assumptions the constant $1/4$ is in fact necessary and hence cannot be improved for both \affk $\textrm{ }$ and \alnk. 
\begin{theorem}[\textbf{Necessary conditions for \affk \text{ }and \alnk \text{ }to consistently recovery the planted partition}]
\label{thm:necessary}
Fix $\zeta > 0.$ Let $\beta$ be any sequence of bandwidth parameters and let $g$ be the Gaussian kernel with bandwidth parameter $4 \beta^2 +\zeta$. For any $\epsilon > 0$, there exists $\trueMixMsr \in \prob^2_{2}$ such that
\begin{equation}
\label{eq:necc_conds}
    \mathbb{P}_{X_n}\left(\rho(\trueCmpMsr_1, \trueCmpMsr_{2}) > 4\sup \limits_{\substack{x \in \sample_n}} \rho(\compKDEGamma_x, \estCmpMsr_{\sigma^*(x), \sigma^*})  - \epsilon\right)\stackrel{n \rightarrow \infty}{\longrightarrow} 1
\end{equation}and the algorithms $\mathcal{A}_{\textrm{FFK}}$ and $\mathcal{A}_{\textrm{LNK}}$ with kernel $g$ fail to recover the planted partition $\truePart$ with probability approaching $\frac{1}{2}$ and 1, respectively, as $n\to\infty$.
\end{theorem}
The proofs for the results appear in the supplementary. For the kernel k-center problem, we can indeed show that the constant in the sufficient conditions (\ref{eq:suff_conds}) can further be improved to $1/3$ when $K=2.$ However, we believe that for any arbitrary $K$, the conditions provided in (\ref{eq:suff_conds}) cannot be further improved. This can be shown for a linear kernel and we leave the more general case of the Gaussian kernel as a conjecture. 
 Our results not only show that certain kernel-based clustering algorithms can exploit separability to recover the planted clustering but also clearly show that under no additional assumptions very strong separability conditions are necessary to obtain recovery guarantees for kernel-based clustering. Furthermore, due to reasons of identifiability, simple separation conditions between the component distributions do not suffice to derive consistent recovery guarantees. For instance consider a simple example of a mixture distribution shown in Figure \ref{fig:identifiability_example}. As $\trueCmpMsr_{2,2}$ moves arbitrarily far from the remaining distributions, the distance between the two component distributions, $\trueCmpMsr_1, \trueCmpMsr_2$ becomes arbitrarily far. However, without additional assumptions, it is not possible for a clustering algorithm to recover the desirable clustering even if we see infinite amount of data. Therefore, the separability conditions on the component distributions are necessarily dependent on the geometric properties of the distribution and not merely on the sample size or the dimension of the input space as it often is in the parametric setting. Our results, providing necessary and sufficient recovery conditions for kernel-based data clustering algorithms (Theorems \ref{thm:impossibility}, \ref{thm:consistency_sufficient} , and \ref{thm:necessary}), are obtained by analyzing an equivalent density/distribution clustering procedure which is considerably easier to analyze. Specifically, this equivalence allows us to exploit the metric geometry of the space of probability measures on the Euclidean space. We now describe this relationship between kernel-based data clustering and kernel-based density clustering.
\begin{figure*}
    \centering
    \includegraphics[width=\textwidth]{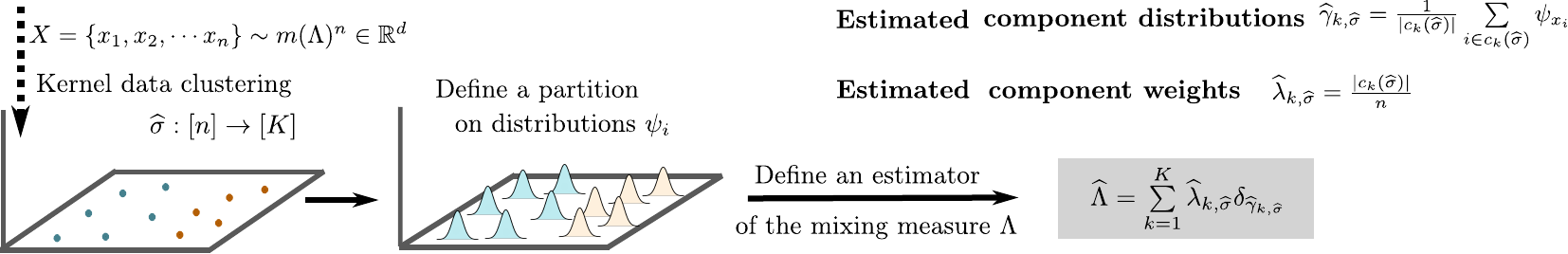}
    \caption{Illustration of the estimation procedure defined with respect to a kernel clustering algorithm. Any Gaussian kernel clustering algorithm can be used to define a partition on component density functions $\myCurls{\psi_i}_{i=1}^n$ which can in turn be used to define an estimator of the mixing measure $\trueMixMsr.$}
    \label{fig:estimation}
\end{figure*}
\section{EQUIVALENCE BETWEEN KERNEL-BASED DATA CLUSTERING AND DISTRIBUTION CLUSTERING}
\label{sec:kde_clustering_equivalence}
In this section, we present a density clustering procedure and describe its close relationship to kernel-based data clustering. Given a finite sample $\sample$, the density clustering procedure clusters the component probability distributions ($\compKDEGamma_i$) of the kernel density estimate with respect to $X$ using MMD as the metric between the distributions. This procedure is illustrated in Figure \ref{fig:kde_equivalence}. As shown in Figure \ref{fig:kde_equivalence}, the partition obtained by this density clustering procedure can be used to define a partition on the sample $X$. This partition can alternatively be obtained by using a simple kernel-based data clustering procedure. We now describe this density clustering procedure, which we denote by \akde.

\textbf{Kernel-based density clustering \akde.}
Consider Gaussian kernel $g_\zeta$ for some $\zeta>0$. Given sample $X \sim \trueMsr^n$:
\begin{itemize}[topsep=0pt]
    \item  Estimate the density of $\trueMsr$ by $\widehat{f} = \frac{1}{n}\sum \limits_{i=1}^n \widetilde{f}_i$ as in (\ref{eqn:kde}) with a bandwidth parameter $\beta>0.$
    
    \item Consider MMD corresponding to the Gaussian kernel $g_{\zeta}$ as the metric between the distributions. Cluster the probability distributions $\myCurls{\compKDEGamma_i}_{i=1}^n$ corresponding to $\myCurls{\widetilde{f}_i}_{i=1}^n$ by means of a distance based clustering algorithm (for example, k-means) to obtain a partition function $\widehat{\sigma}$.
\end{itemize}
This procedure is also illustrated in Figure \ref{fig:kde_equivalence}. We show that for appropriately chosen bandwidth parameters, any kernel-based data clustering algorithm can be equivalently formulated as a density clustering procedure (\akde). Recall that $\beta$ and $\zeta$ are the bandwidth parameters of the Gaussian kernels used in \akde~for kernel density estimation and for defining the MMD respectively. Then, let $g: \domain \times \domain \rightarrow \R$ be the Gaussian kernel with bandwidth parameter $4 \beta^2 + \zeta$. The following lemma shows that the maximum mean discrepancy between the component distributions ($\psi_{i}$) is closely related to kernel evaluations on the input data.
\begin{lemma}[\textbf{MMD between components is closely related to kernel evaluations between input data.}]
\label{lemma:mmd_kernel_eqv}
Given any sample $X \in \R^d$, let the component KDE distributions $(\psi_i)$ be defined in the usual way. For all $x_i, x_j \in X$, 
\begin{equation*}
    \rho^2(\psi_i, \psi_j) = C_{\beta, \zeta, d}(1 - g(x_i, x_j))
\end{equation*}
where $C_{\beta, \zeta, d}$ is a constant dependent on the bandwidths $\beta, \zeta$ and the input dimension $d$.
\end{lemma}
We obtain this result by explicitly computing the MMD between the component distributions. Theorem \ref{thm:kde_equivalence} is then an immediate consequence of Lemma \ref{lemma:mmd_kernel_eqv}, which states that every kernel based data-clustering algorithm can equivalently be formulated as a kernel-based density clustering procedure (see Figure \ref{fig:kde_equivalence}).
\begin{theorem}[\textbf{Equivalence between kernel data-clustering and \akde}]
\label{thm:kde_equivalence}
Any Gaussian kernel-based (data) clustering algorithm can equivalently be formulated as a clustering of the component KDE distributions with respect to the MMD metric corresponding to a Gaussian kernel for appropriately chosen bandwidth parameters. 
\end{theorem}
This simple result is consequential for practical considerations such as in the choice of bandwidth parameter for kernel data clustering (see Section \ref{sec:discussion}) as well as for theoretical considerations. As it turns out, the density clustering procedure ($\mathcal{A}_{\textrm{KDE}}$) of the component KDE distributions can be used to define an estimator of the true mixing measure, that is, true component distributions and the corresponding weights. The equivalence between the two procedures, therefore, allows us to derive consistency guarantees for the estimators by analyzing the corresponding kernel-based clustering algorithms.

\section{CONSISTENCY OF ESTIMATING MIXTURE MODELS}
\label{sec:estimation_and_Bayes_partition}
\textbf{Estimation procedure.} By an estimation procedure, we refer to any algorithm that takes a sample $\sample$ drawn according to some mixing measure $\trueMixMsr$, that is, $\sample \sim m(\trueMixMsr)^n$ and provides an estimate $\estMixMsr$ of $\trueMixMsr$.

\textbf{Identifiability.} Identifiability is also a key issue for estimation. Similar to our analysis of non-parametric clustering, we circumvent an \textit{explicit} analysis of identifiability. Moreover, in the preceding discussion, identifiability is defined as a deterministic property of a set of mixing measures. We introduce a \textbf{\textit{statistical}} notion of identifiability which can be defined as a property of either a mixing measure or a set of mixing measures. Additionally, in contrast to identifiability, statistical identifiability is defined with respect to an algorithm and therefore, it is a more intuitive and natural definition in the analysis of estimation procedures. Intuitively, the set of all mixing measures which are identifiable with respect to an estimation procedure $\mathcal{E}$ encodes the \textit{\textbf{inductive bias}} of $\mathcal{E}.$ 

\begin{definition}[\textbf{Statistical identifiability}]
 Let $\varrho$ be some metric defined on the space of all mixing measures $\mixingSpace$. A mixing measure $\trueMixMsr$ is statistically identifiable with respect to an estimation procedure $\mathcal{E}$ if the sequence of mixing measures $\left \{ \estMixMsr_n = \mathcal{E}(X_n) \right \}$ converges in probability to $\trueMixMsr$, where given $X_n \sim m(\trueMixMsr)^n$. 
 
 Furthermore, a set of mixing measures $\mixset \subset \mixingSpace$ is said to be statistically identifiable with respect to estimation procedure $\mathcal{E}$ if every mixing measure $\trueMixMsr \in \mixset$ is statistically identifiable with respect to $\mathcal{E}$. 
\end{definition}
\textbf{Remark.} The convergence of the mixing measures can be defined with respect to any metric on $\mixingSpace$. In our results, we show convergence with respect to the Wasserstien distance between mixing measures (see the supplementary for a definition).

\textbf{Estimation procedure based on kernel-based data clustering.}
We describe the procedure to define an estimator of the true mixing measure $\trueMixMsr.$ This procedure is illustrated in Figure \ref{fig:estimation}. As usual, for some $\beta, \zeta > 0$, denote the Gaussian kernel with bandwidth parameter $4 \beta^2 + \zeta >0$ by $g$. The component probability distributions of the KDE $\psi_i$ are also defined in the usual way with respect to the bandwidth parameter $\beta > 0$. Given a sample $\sample_n \sim m(\trueMixMsr)^n$, 
\begin{enumerate}[label=(\alph*),topsep=0pt]
    \item  By means of a kernel-based data clustering procedure, with respect to $g$, obtain a partition $\widehat{\sigma}:[n] \rightarrow [K]$ of $X_n$.
    \item Use $\widehat{\sigma}$ to define a partition of component KDE distributions $\myCurls{\psi_i}_{i=1}^n$. 
    \item The estimator is defined as $\estMixMsr_n = \sum \limits_{i=1}^{K} \estCmpwgt_{k, \widehat{\sigma}} \delta_{\estCmpMsr_{k, \widehat{\sigma}}}$, where $\estCmpMsr_{k, \widehat{\sigma}} = \frac{1}{\abs{c_k}} \sum_{x_i \in c_k} \compKDEGamma_i$ and $\estCmpwgt_{k, \widehat{\sigma}} = \frac{\abs{c_k}}{n}$.

\end{enumerate}
Let \ectr, \effk, and \elnk \textrm{ }denote the estimation procedures corresponding to the kernel data clustering algorithms, \actr, \affk, and \alnk\textrm{ }respectively: the estimation procedure that uses the respective kernel clustering algorithm to obtain a partition $\widehat{\sigma}$ in (a). Theorem \ref{thm:identifiability_lnk} then immediately follows from the recovery guarantees for the corresponding kernel-based clustering algorithms (Theorem \ref{thm:consistency_sufficient}) and the equivalence between kernel data clustering and density clustering established in Theorem \ref{thm:kde_equivalence}. We show that any mixing measure satisfying the conditions provided in (\ref{eq:suff_conds}) is statistically identifiable with respect to the estimation procedures corresponding to \actr, \affk, and \alnk.
\begin{theorem}[\textbf{Statistical identifiability with respect to \ectr, \effk, and \elnk}]
\label{thm:identifiability_lnk}
Let $\zeta$ and $\beta$ be bandwidth parameters satisfying the conditions provided in Theorem \ref{thm:consistency_sufficient}. Then any $\trueMixMsr \in \prob^2_{K}$ satisfying the conditions provided in (\ref{eq:suff_conds}) is statistically identifiable with respect to \ectr, \effk, and \elnk.
\end{theorem}


\textbf{Estimating the Bayes partition.}
For theoretical considerations, it might be of interest to analyze conditions under which kernel-clustering algorithms can consistently estimate the Bayes partition. Given a finite sample $\nSample$, let $\widehat{\sigma}$ denote the partition generated by a kernel clustering algorithm $\mathcal{A}.$ We can define an estimator of the Bayes partition function $\widehat{\sigma}_b:\R^d \rightarrow [K]$ in the natural way:

\begin{equation}
\label{eqn:reassignment_step}
    \widehat{\sigma}_b(x) = \argsup \limits_{k \in [K]} \sum \limits_{j: \widehat{\sigma}(j) = k} G_{\beta}(x, x_j) \stackrel{(*)}{=} \argsup \limits_{k \in [K]} \estCmpwgt_{k, \widehat{\sigma}} \widehat{f}_{k, \widehat{\sigma}}(x)
\end{equation}
where $(*)$ follows from Lemma \ref{lemma:mmd_kernel_eqv}. Due to the equivalence between kernel clustering and density-based clustering, we can show that if a kernel-based algorithm $\mathcal{A}$ can consistently recover the planted partition, then by means of a single reassignment step given by (\ref{eqn:reassignment_step}), the algorithm  consistently recovers the Bayes partition. 

\textbf{Exceptional set.} Given $\trueMixMsr = \sum_{k \in [K]} \trueCmpwgt_k \delta_{\trueCmpMsr_k}$, for any $t > 0$, we define the exceptional set
\begin{equation*}
\label{eq:exceptional_set}
    E(t) = \bigcup \limits_{k \neq k'} \myCurls{x \in \R^d : \abs{\trueCmpwgt_k f_k(x) - \trueCmpwgt_{k'} f_{k'}(x)} \leq t}.
\end{equation*}
\begin{theorem}[\textbf{Estimating the Bayes partition}]
\label{thm:Bayes_consistency}
Let $\zeta$, and $\beta$ be bandwidth parameters satisfying the conditions provided in Theorem \ref{thm:consistency_sufficient}. Let $\trueMixMsr \in \prob^2_{K}$ satisfying the conditions provided in (\ref{eq:suff_conds}). For $\nSample \sim m(\trueMixMsr)^{n}$ and let $\widehat{\sigma}_{b, n}$ be the partition function obtained by \actr, \affk \textrm{ or } \alnk \textrm{ } followed by the reassignment step in (\ref{eqn:reassignment_step}). Then, w.h.p over the samples, there exists a sequence $\myCurls{t_n} \stackrel{n \rightarrow \infty}{\longrightarrow} 0$ such that $\widehat{\sigma}_n(x) = \BayesPart(x)$ for all $x \in \R^d - E_0(t_n).$
\end{theorem}

\section{DISCUSSION AND FUTURE WORK}
\label{sec:discussion}
We show in this work that certain kernel-based clustering algorithms can exploit separability conditions to overcome identifiability. Our results also show that strong separability conditions are indeed necessary for provable recovery guarantees for clustering methods under non-parametric conditions. To further elaborate, we highlight a conceptually interesting insight from our results, which is surprising on the first glance. Even though kernel-based FFk-means++, which is a relaxation of the NP-Hard kernel k-means can provably recover the true clusters under the sufficient separability conditions (Theorem \ref{thm:consistency_sufficient}), our impossibility result (Theorem \ref{thm:necessary}) shows that the NP-Hard kernel k-means algorithm fails to (provably) do so. This clearly shows that for better recovery guarantees for a clustering algorithm $\mathcal{A}$, in the non-parametric setting, it is essential to thoroughly characterize the inductive bias of the $\mathcal{A}$, that is, the set of mixing measures for which $\mathcal{A}$ can recover the true clustering.

We also established a key connection between kernel data clustering and distribution clustering when using Gaussian kernels and MMD as a metric between the distributions. As a consequence, we can interpret any standard Gaussian kernel clustering algorithm as a distribution clustering procedure. This is particularly useful in theoretical analysis since, for instance, we can analyze kernel clustering algorithms by analyzing the corresponding distribution clustering procedure and vice versa. This connection could also have practical implications on matters such as bandwidth selection for kernel clustering.

\textbf{Extending our results beyond the Gaussian kernel.} We believe that the relationship between kernel data clustering and density clustering can indeed be established for a larger class of kernel functions. For instance, choosing kernel functions from conjugate families is one way in which the analysis could possibly be extended to other kernels, that is, choosing the MMD kernel function as the conjugate prior of the kernel function used for density estimation. It would also be of significant interest to characterize the class of kernels for which the equivalence can be established. However, a detailed study in this direction is reserved for future work. 

\textbf{Bandwidth.} There is little to no literature that provides a systematic approach to bandwidth selection for kernel-based clustering. In contrast to kernel clustering, bandwidth selection is a well studied problem in the context of kernel density estimation \citep{gine2002rates, einmahl2005uniform, goldenshluger2011bandwidth, chacon2013data}. By appropriating bandwidth selection strategies from this work, we provide the following guidance in \textbf{bandwidth selection for kernel-based data clustering}. As it would be expected, our analysis suggests that the bandwidth parameter used for kernel-clustering ($4 \beta^2 + \zeta$) needs to decrease with $n$ since our sufficient conditions for recovery require that $\beta \stackrel{n \rightarrow \infty}{\longrightarrow} 0$. Interestingly, however, it suggests that the bandwidth parameter can asymptotically remain non-zero since $\zeta $ is chosen to be a fixed parameter greater than $0$. We note that these conditions are asymptotic and a more thorough analysis of the convergence rates of the estimators is necessary to provide the rate at which the bandwidth
needs to reduce with sample size. Moreover, the range of the bandwidth parameter, which depends on the constant terms, could be be data-dependent. We conducted few small-sample experiments, and observed that the dependence of clustering performance on bandwidth is complex and requires more thorough investigation. We leave this analysis for future work.

\subsubsection*{Acknowledgements}

This work has been supported by the Baden-W{\"u}rttemberg Stiftung Eliteprogram for Postdocs through the project ``Clustering large evolving networks'', the International Max Planck Research School for Intelligent Systems (IMPRS-IS), the German Research Foundation through the Cluster of Excellence
“Machine Learning – New Perspectives for Science" (EXC 2064/1 number 390727645), and the Tübingen AI Center (FKZ: 01IS18039A).

\printbibliography

\clearpage
\onecolumn

\appendix
\section{Equivalence between Kernel-based data clustering and Kernel-based density clustering.}

\subsection{Proof of Lemma 1}

\begin{lemma}[\textbf{MMD between components is closely related to kernel evaluations between input data.}]
\label{lemma:mmd_kernel_eqv_supp}
Given any sample $X \in \R^d$, let the component kde distributions $(\psi_i)$ be defined in the usual way. For all $x_i, x_j \in X$, 
\begin{equation*}
    \rho^2(\psi_i, \psi_j) = C_{\beta, \zeta, d}(1 - g(x_i, x_j))
\end{equation*}
where $C_{\beta, \zeta, d}$ is a constant dependent on the bandwidths $\beta, \zeta$ and the input dimension $d$.
\end{lemma}
\begin{proof}
Squared MMD $\rho^2(\compKDEGamma_i, \compKDEGamma_j)$ with respect to the Gaussian kernel $g_{\zeta}$ can be decomposed as follows:
\begin{equation}
\label{eqn:MMD2_decomp}
    \rho^2(\compKDEGamma_i, \compKDEGamma_j) = \mynorm{\mu_{\compKDEGamma_i}}_{\mathcal{H}_{g_{\zeta}}}^2 + \mynorm{\mu_{\compKDEGamma_j}}_{\mathcal{H}_{g_{\zeta}}}^2 - 2 \inner{\mu_{\compKDEGamma_i}, \mu_{\compKDEGamma_j}}_{\mathcal{H}_{g_{\zeta}}},
\end{equation}
where $\mu_{\compKDEGamma_j}$ denotes the kernel mean embedding of $\compKDEGamma_i$ with respect to the Gaussian kernel function $g_{\zeta}$ which can be computed in closed form as shown in (\ref{eqn:kme_computed}). 
\begin{equation}
\begin{aligned}
\label{eqn:kme_computed}
    \mu_{\compKDEGamma_j}(\cdot) &=  \int \limits_{\R^d} \frac{1}{(2 \pi \beta^2)^{d/2}} \exp{\myBracs{\frac{-\norm{x - \cdot}^2}{\zeta}}} \exp{\myBracs{\frac{-\norm{x_j - \cdot}^2}{2 \beta^2}}} dx \\
    &=(\frac{\zeta}{\zeta + 2 \beta^2})^{d/2} \exp{\myBracs{- \frac{\norm{ x_j - \cdot }^2}{2 \beta^2 + \zeta}}}
\end{aligned}
\end{equation}
By means of theorem \ref{thm:RKHS_characterization} which provides a spectral characterization of the Gaussian RKHS and the inner-product within, we compute $\inner{\mu_{\compKDEGamma_i}, \mu_{\compKDEGamma_j}}_{\mathcal{H}_{g_{\zeta}}}$, $\forall i, j \in [n]$. The computation uses the closed form expressions of Fourier transforms of the kernel function and the kernel mean embeddings of the component kde distributions given in (\ref{eq:fourier_transforms}). The closed form expression for the inner product between the kernel mean embeddings of any two component kde distributions is given in Equation (\ref{eqn:inner_kme}).

\begin{equation}
\begin{aligned}
\label{eq:fourier_transforms}
     \mathcal{F}[g_{\zeta}](\omega)  &= (\frac{\zeta}{2})^{d/2} \exp{\myBracs{\frac{-\norm{\omega}^2 \zeta}{4}}}. \\
     \mathcal{F}[\mu_{\compKDEGamma_i}](\omega)  &= (\frac{\zeta}{2})^{d/2} \exp{\myBracs{\frac{-\norm{\omega}^2 (2 \beta^2 + \zeta)}{4}}} \exp{ \myBracs{ \textbf{i} \sum \limits_{l \in [d]} x_i^l \omega^l}}, \\
\end{aligned}
\end{equation}
where $\textbf{i}$ denotes the imaginary unit and satisfies $\textbf{i}^2 = -1.$
\begin{equation}
\begin{aligned}
    \label{eqn:inner_kme}
    \inner{\mu_{\compKDEGamma_i}, \mu_{\compKDEGamma_j}}_{\mathcal{H}_{g_{\zeta}}} &= \frac{1}{(2 \pi )^{d/2}}\int \frac{\mathcal{F}[\mu_{\compKDEGamma_i}](\omega)\overline{\mathcal{F}[\mu_{\compKDEGamma_i}](\omega)}}{\mathcal{F}[g_{\zeta}](\omega)} d\omega   \\
    &= \left (\frac{\zeta}{4 \beta^2 + \zeta}\right )^{d/2} \exp{ \myBracs{\frac{-\norm{x_i - x_j}^2}{4 \beta^2 + \zeta}}}
    \end{aligned}
\end{equation}
Substituting the values of $\inner{\mu_{\compKDEGamma_i}, \mu_{\compKDEGamma_j}}_{\mathcal{H}_{g_{\zeta}}}$ for any $i, j \in [n]$ we obtain
\begin{equation}\label{eq:mmd_for_gaussians}
    \rho^2(\compKDEGamma_i, \compKDEGamma_j) = 2 \left (\frac{\zeta}{4 \beta^2 + \zeta}\right )^{d/2} (1 - g(x_i, x_j))
\end{equation}
\end{proof}


The following result given by \citet{kimeldorf1970correspondence, wendland2004scattered} provides a spectral characterization of the RKHS corresponding to any translation-invariant kernel.

\begin{theorem}[\textbf{Spectral characterization of RKHS. \citep{kimeldorf1970correspondence, wendland2004scattered}}]
\label{thm:RKHS_characterization}
Let $k$ be a translation-invariant kernel on $\R^d$ such that $k(x,y) := \psi(x-y)$ where $\Phi \in C(\R^d) \cap L_1(R^d)$. Then the corresponding RKHS $\rkhs$ is given by
\begin{equation} \label{eq:RKHS-shift-invariant}
\rkhs = \left\{ f \in L_2( \R^d ) \cap C(\R^d):\ \RKHSnorm{f}^2 = \frac{1}{ (2\pi)^{d/2} }  \int \frac{ |\mathcal{F}[f](\omega)|^2 }{\mathcal{F}[\psi](\omega)} d\omega < \infty \right\}, 
\end{equation}
where $| \cdot |$ denotes the magnitude of the enclosed quantity and $\mathcal{F}[f](\omega)$ denotes the Fourier transform of the function $f$. The inner product on $\rkhs$ is defined as $\inner{f, g}_{\rkhs} =  \frac{1}{ (2\pi)^{d/2} } \int \frac{\mathcal{F} [f] (\omega) \overline{\mathcal{F}[g](\omega)}}{\mathcal{F} [\psi](\omega) }  d\omega,\quad f,g \in \rkhs, $
where $\overline{\mathcal{F}[g](\omega)}$ denotes the complex conjugate of $\mathcal{F}[g](\omega)$.
\end{theorem}
\subsection{Proof of Theorem 4}
\textbf{Theorem \ref{thm:kde_equivalence}} immediately follows from Lemma \ref{lemma:mmd_kernel_eqv}. 
For any data clustering algorithm with respect to the Gaussian kernel $\eta > 0$, decompose $\eta$ into any two positive quantities $\beta, \zeta > 0$ satisfying $\eta = 4 \beta^2 + \zeta$. Due to Lemma \ref{lemma:mmd_kernel_eqv}, the kernel clustering algorithm equivalently defines a clustering of the component kde distributions $\myCurls{\compKDEGamma_i}_{i=1}^n.$ 


\section{Algorithms}
For completeness, we briefly describe the kernel-based clustering algorithms (\akmn, \actr, \affk, and \alnk) here. In each of the algorithms, we describe the standard kernel data clustering procedure as well as the equivalent kernel density clustering procedures (see Theorem \ref{thm:kde_equivalence}). The component kde distributions are defined in the usual way with respect to the bandwidth parameter $\beta>0$ and $\rho$ is defined with respect to the Gaussian kernel with bandwidth parameter $\zeta > 0.$
\newpage
\subsection{Kernel k-means (\akmn)}\label{sec:akm}
\par\noindent\rule{\textwidth}{0.4pt}
\vspace{-7mm}
\begin{center}
    \textbf{Algorithm - Kernel k-means}
\end{center}
\vspace{-5mm}
\par\noindent\rule{\textwidth}{0.4pt}
\begin{itemize}
    \item Given: A sample $\nSample \subset \R^d$ and for some $\beta, \zeta > 0$ the Gaussian kernel function $g:\R^d \times \R^d \rightarrow \R$ with bandwidth parameter $4 \beta^2 + \zeta.$
    
    \item Find the partition 
\end{itemize}
     
    \begin{equation}
\label{eq:kmeans_clustering}
    \widehat{\sigma} = \argmax \limits_{\sigma:[n] \rightarrow [K] } \sum \limits_{k \in [K]} \sum \limits_{i, j \in c_k} g(x_i, x_j) = \argmin \limits_{\sigma:[n] \rightarrow [K] } \sum \limits_{k \in [K]} \sum \limits_{i \in c_k} \rho(\mu_{\compKDEGamma_i} , \frac{1}{\abs{c_k}} \sum \limits_{j \in c_k} \mu_{\compKDEGamma_j})^2
\end{equation}
\subsection{FFk-means++ (\affk)}
\label{sec:A_FFK}
\par\noindent\rule{\textwidth}{0.4pt}
\vspace{-7mm}
\begin{center}
    \textbf{Algorithm - Farthest first Kernel k-means ++}
\end{center}
\vspace{-5mm}
\par\noindent\rule{\textwidth}{0.4pt}
\vspace{2mm}
\textbf{Phase one: Initializing the centers}
\begin{itemize}
     \item Given: A sample $\nSample \subset \R^d$ and for some $\beta, \zeta > 0$ the Gaussian kernel function $g:\R^d \times \R^d \rightarrow \R$ with bandwidth parameter $4 \beta^2 + \zeta.$
    
    \item Choose an initial center $c_1$ uniformly at random and set $C = \myCurls{c_1}.$
    
    \item While $t < K:$
    
    \begin{itemize}
    \item let $C = \myCurls{c_1, c_2, \cdots c_{t-1}}$ be the current set of centers,
        \item for each $x \in X $, compute $d(x) = \min \limits_{c \in C} k(x, c) = \max \limits_{c \in C} \rho(\compKDEGamma_{x}, \compKDEGamma_{c})$
        \item pick the new center $c_{t} = \argmax \limits_{x \in X} d(x),$ and set $C = C \cup \myCurls{c_t}.$
    \end{itemize}
    \item For each $k \in [K]:$
    \begin{itemize}
        \item set 
        \vspace{-5mm}
        \begin{align*}
            C_k &= \myCurls{x \in \sample: k(x, c_k) \geq k(x, c_{k'}) \; \forall k \neq k' \in [K]} \\
            &= \myCurls{x \in \sample: \rho(\compKDEGamma_x, \compKDEGamma_{c_k}) \leq \rho(\compKDEGamma_x, \compKDEGamma_{c_k'}) \; \forall k \neq k' \in [K]}
        \end{align*}
    \end{itemize}
\end{itemize}

\textbf{Phase two: Standard kernel k-means algorithm}

\begin{enumerate}
 
    \item For each $k \in [K]$, set $C_k = \myCurls{x \in \sample: \textrm{condition (\ref{eq:kmeans_cond_supp}) holds}}$
    \begin{equation}
    \label{eq:kmeans_cond_supp}
          \frac{1}{\abs{C_k}^2}\sum \limits_{y, z \in C_k} k(x, z) - \frac{1}{\abs{C_k}} \sum \limits_{y \in C_k} k(y, x) \leq \frac{1}{\abs{C_{l}}^2}\sum \limits_{y, z \in {l}} k(y, z) - \frac{1}{\abs{C_{l}}} \sum \limits_{y \in C_{l}} k(y, x) \; \forall l \neq k \in [K].
    \end{equation}
    \begin{equation}
    \label{eq:kmeans_cond_supp_density}
  (\ref{eq:kmeans_cond_supp}) \iff \rho(\compKDEGamma_x, \frac{1}{\abs{C_k}}\sum \limits_{x' \in C_k} \compKDEGamma_{x'}) \leq \rho(\compKDEGamma_x, \frac{1}{\abs{C_{l}}} \sum \limits_{x' \in C_l} \compKDEGamma_{x'}) \; \;  \forall l \neq k \in [K].
    \end{equation}
        
\item Repeat step (1) until convergence, that is, the set of centers $C$ do not change anymore.

\end{enumerate}


\subsection{Kernel K-center(\actr)}
\par\noindent\rule{\textwidth}{0.4pt}
\vspace{-5.5mm}
\begin{center}
    \textbf{Algorithm - Kernel K-center}
\end{center}
\vspace{-3.5mm}
\par\noindent\rule{\textwidth}{0.4pt}
\begin{itemize}
    \item Given: A sample $\nSample \subset \R^d$ and for some $\beta, \zeta > 0$ the Gaussian kernel function $g:\R^d \times \R^d \rightarrow \R$ with bandwidth parameter $4 \beta^2 + \zeta.$
    \item Find the partition 
\end{itemize}
   \begin{align*}
\label{opt_problem:k_center_clustering}
    \widehat{\sigma} & = \argmax \limits_{\sigma:[n] \rightarrow [K] } \inf \limits_{l \in [n]} \frac{-1}{\vert c_k^l \vert^2 } \sum \limits_{i, j \in c_k^l}  k(x_i,x_j) + \frac{1}{\vert c_k^l \vert } \sum \limits_{ i \in c_k^l }k(x_i,x_l) \\
    &  = \argmin \limits_{\sigma:[n] \rightarrow [K] } \max \limits_{i \in [n]} \rho(\compKDEGamma_i, \estCmpMsr_{\sigma(i), \sigma})
\end{align*}
\subsection{Agglomerative hierarchical clustering (\alnk)}
\label{sec:A_LNK}
 Given a sample $\nSample \subset \R^d$ and a similarity function $S:\R^d \times \R^d \rightarrow \R$, hierarchical clustering algorithms seek to generate a cluster tree (dendrogram) establishing a hierarchy of relationships between the elements of the sample. Aggolomerative methods, in contrast to divisive methods, seek a bottom up approach, starting out with each point as its own cluster and progressively combining them into larger clusters until there is a single cluster that contains all the elements of the sample $\sample.$ The criterion for merging hinges on the underlying similarity function, which in our case is the kernel matrix computed on the sample for a given kernel function $k:\R^d \times \R^d \rightarrow \R.$ We discuss two of the popular hierarchical clustering algorithms that exist in literature: \textbf{single linkage} and \textbf{complete linkage} methods. The distinguishing factor across the two methods is the choice of the criterion $C$ used to merge any two clusters $c, c' \subset X$ ($c \cap c'= \emptyset$), which are given below in \ref{eqn:linkages}.
\begin{equation}
    \label{eqn:linkages}
  C(c, c') =  \underbrace{\vphantom{\sum \limits_{x \in c, y \in c'} \frac{k(x,y)}{\vert c \vert \vert c' \vert}} \max \limits_{x \in c, y \in c'} k(x, y) = \min \limits_{x \in c, y \in c'} \rho(\compKDEGamma_x, \compKDEGamma_y)}_{\textbf{Single linkage}} \textrm{, } \quad  \textrm{and} \quad \underbrace{\vphantom{\sum \limits_{x \in c, y \in c'} \frac{k(x,y)}{\vert c \vert \vert c' \vert}} \min \limits_{x \in c, y \in c'} k(x, y) = \max \limits_{x \in c, y \in c'} \rho(\compKDEGamma_x, \compKDEGamma_y)}_{\textbf{Complete linkage}}.
\end{equation}
 By substituting the different criterion $C(c, c')$ to merge any two clusters $c, c'$  in Algorithm \ref{alg:single_linkage}, we obtain variants of the corresponding algorithms.
\begin{algorithm}
\SetAlgoLined
\textbf{Given:} A sample $\nSample \subset \R^d$ and for some $\beta, \zeta > 0$ the Gaussian kernel function $g:\R^d \times \R^d \rightarrow \R$ with bandwidth parameter $4 \beta^2 + \zeta$ \;
Let $\mathcal{S} = \myCurls{s_1, \dots s_n}$ be a collection of singleton trees with the root node of $s_i = \myCurls{i}.$

\While{$\vert \mathcal{S} \vert > 1$}{
Let $s_{q}, s_{r} \in \mathcal{S}$ be the pair of trees such that $C( root(s_q),root(s_r))$ is maximal \;

Generate $s_{qr}$ s.t, $root(s_{qr}) = root(s_q) \cup root(s_r)$, $left, right(s_{qr}) = s_q, s_r$ \;

Add $s_{qr}$ and remove $s_q$ and $s_r$ from $\mathcal{S}$ \;
}
$\hat{\sigma} \gets$ Partition function obtained by cutting the only element in $\mathcal{S}$, a dendrogram at a level such that the resulting partition contains $K$ clusters \;

  \Return{$\hat{\sigma}$ \;}
 \caption{Agglomorative hierarchical kernel-clustering.}
 \label{alg:single_linkage}
\end{algorithm}

\section{Impossibility of recovery by kernel k-means(Proof of Theorem 1)}
\begin{proof}
Fix the kernel bandwidth parameter $\zeta>0$. Consider the following example in $\mathbb{R}$, where $\mathcal{U}([a,b])$ denotes the uniform distribution on the real interval $[a,b]$. Let
\begin{equation}
\gamma_1 =m\left(\frac{1}{2}\mathcal{U}([-\epsilon,\epsilon])+\frac{1}{2}\mathcal{U}([r-\epsilon,r+\epsilon])\right)
\end{equation}
and
\begin{equation}
\gamma_2 =\mathcal{U}([Dr-\epsilon,Dr+\epsilon]).
\end{equation}The mixing measure is given by $\Lambda=\lambda_1\gamma_1+\lambda_2\gamma_2$. The constants $D\gg 2\gg r\gg\epsilon$ and $\lambda_1\gg\lambda_2$ are to be chosen later. The idea is that the interval $[Dr-\epsilon,Dr+\epsilon]$ is separated from the rest of the distribution via a large constant $D$, but the points in $[Dr-\epsilon,Dr+\epsilon]$ will nevertheless be clustered with the points in $[r-\epsilon,r+\epsilon]$ because $\lambda_2$ is so small. We first show that $\Lambda$ satisfies the condition in the theorem, namely that
\begin{equation}
\label{eq:proof_thm3_criterion}
\frac{\rho^2(\trueCmpMsr_1, \trueCmpMsr_{2}) }{\sup \limits_{\substack{x \in \sample_n}} \rho^2(\compKDEGamma_x, \estCmpMsr_{\sigma^*(x), \sigma^*})}>K^2.
\end{equation}Therefore, consider the numerator, which is simply the squared MMD between $\gamma_1$ and $\gamma_2$. We have
\begin{equation*}
\begin{split}
\rho^2(\trueCmpMsr_1, \trueCmpMsr_{2})  &= \mathbb{E}_{X\sim\gamma_1,\tilde X\sim\gamma_1}g(X,\tilde X)+\mathbb{E}_{Y\sim\gamma_2,\tilde Y\sim\gamma_2}g(Y,\tilde Y)-2\mathbb{E}_{X\sim\gamma_1,Y\sim\gamma_2}g(X,Y)\\[6pt]
&\geq \frac{1}{(2\epsilon)^2}\int_{[-\epsilon,\epsilon]^2}e^{-|x-y|^2/\zeta}\,dx\,dy-\frac{2}{(2\epsilon)^2}\int_{[-\epsilon,\epsilon]^2}e^{-|(D-1)r+x-y|^2/\zeta}\,dx\,dy.
\end{split}
\end{equation*}
At this point, assume that $\epsilon$ is sufficiently small compared to the kernel bandwidth parameter $\zeta$, namely that $4\epsilon^2<\eta$. This allows us to lower bound the first integral by $\frac{1}{e}$. Similarly, choosing $D$ large enough in comparison to $r$ allows us to make the second term arbitrarily small, whence we conclude that
\begin{equation*}
\rho^2(\trueCmpMsr_1, \trueCmpMsr_{2}) \geq \frac{1}{e}-\frac{1}{2e}\geq \frac{1}{2e},
\end{equation*}i.e. the numerator is at least $\frac{1}{2e}$. Now consider the denominator, which is the maximum squared MMD between an empirical cluster mean and a sampled point belonging to that cluster. This is at most the squared MMD between any two points belonging to the same cluster
\begin{equation}
\sup \limits_{\substack{x \in \sample_n}} \rho^2(\compKDEGamma_x, \frac{1}{|\sigma^*(x)|}\sum_{y\in \sigma^*(x)}\compKDEGamma_y)\leq \sup \limits_{\substack{x,y \in \sample_n},\sigma^*(x)=\sigma^*(y)}\rho^2(\compKDEGamma_x,\compKDEGamma_y) 
\end{equation}
which can be bound, independently of the sample $X_n$, by 
\begin{equation}
\begin{split}
\rho^2(\compKDEGamma_0,\compKDEGamma_{r+2\epsilon})&=2\sqrt{\frac{\zeta}{4\beta^2+\zeta}}\left(1-e^\frac{-(r+2\epsilon)^2}{4\beta^2+\zeta}\right)\\
&\leq 2\frac{(r+2\epsilon)^2}{\zeta}+o(r^4).
\end{split}
\end{equation}Here $r+2\epsilon$ is the maximum distance of any two points belonging to the same cluster and we used \eqref{eq:mmd_for_gaussians}. Thus, choosing a small $r$ allows us to make the denominator arbitrarily small, and the fraction in \eqref{eq:proof_thm3_criterion} can become larger than any fixed $K^2$.

Now, we show that k-means does w.h.p. not recover the planted partition. The idea is to choose $\lambda_1\gg\lambda_2$. In our sample $X_n$ from $m(\Lambda)$, denote the number of points within $[-\epsilon,\epsilon]$ by $N_1$, the number of points within within $[r-\epsilon,r+\epsilon]$ by $N_2$, and the number of points within $[Dr-\epsilon,Dr+\epsilon]$ by $N_3$. Assume that $n$ is large enough s.t. $N_1,N_2,N_3>0$. We rely on the equivalence between kernel-based data clustering and kernel-based density clustering and directly consider the MMD between component distributions $\psi_{x_i}$ (compare section \ref{sec:akm}). That is we consider k-means w.r.t. the norm $\|\cdot\|^2=<\cdot,\cdot>_{\mathcal{H}_{g_\zeta}}$. The k-means objective of the planted partition is at least 
\begin{equation*}
    N_1\Bigg\|\mu_{\psi_\epsilon}-\frac{N_1\,\mu_{\psi_{-\epsilon}}+N_2\,\mu_{\psi_{r-\epsilon}}}{N_1+N_2}\Bigg\|^2+N_2\Bigg\|\mu_{\psi_{r-\epsilon}}-\frac{N_1\mu_{\psi_{\epsilon}}+N_2\mu_{\psi_{r+\epsilon}}}{N_1+N_2}\Bigg\|^2\geq\frac{N_1 N_2}{N_1+N_2}\Big\|\mu_{\psi_{\epsilon}}-\mu_{\psi_{r-\epsilon}}\Big\|^2+O(\epsilon).
\end{equation*}
Similarly, the k-means objective of the alternative partition where the points in $[r-\epsilon,r+\epsilon]$ and $[Dr-\epsilon,Dr+\epsilon]$ form a cluster is at most
\begin{equation*}
\begin{split}
    & N_1\Big\|\mu_{\psi_{0}}-\mu_{\psi_{2\epsilon}}\Big\|^2+N_2\Bigg\|\mu_{\psi_{r-\epsilon}}-\frac{N_2\mu_{\psi_{r+\epsilon}}+N_3\mu_{\psi_{Dr+\epsilon}}}{N_2+N_3}\Bigg\|^2+N_3\Bigg\|\mu_{\psi_{Dr+\epsilon}}-\frac{N_2\mu_{\psi_{r-\epsilon}}+N_3\mu_{\psi_{Dr-\epsilon}}}{N_2+N_3}\Bigg\|^2\\[6pt]
    \leq &N_1\Big\|\mu_{\psi_{0}}-\mu_{\psi_{2\epsilon}}\Big\|^2+\frac{N_2 N_3}{N_2+N_3}\Big\|\mu_{\psi_{r-\epsilon}}-\mu_{\psi_{Dr+\epsilon}}\Big\|^2+O(\epsilon).
\end{split}
\end{equation*}
Thus, k-means will choose the alternative partition if
\begin{equation}\label{eq:proof_thm_3:k_means_fail}
\begin{split}
    N_1\Big\|\mu_{\psi_{0}}-\mu_{\psi_{2\epsilon}}\Big\|^2+\frac{N_2 N_3}{N_2+N_3}\Big\|\mu_{\psi_{r-\epsilon}}-\mu_{\psi_{Dr+\epsilon}}\Big\|^2+O(\epsilon) &\leq\frac{N_1 N_2}{N_1+N_2}\Big\|\mu_{\psi_{\epsilon}}-\mu_{\psi_{r-\epsilon}}\Big\|^2 \\[6pt]
\iff \frac{\Big\|\mu_{\psi_{r-\epsilon}}-\mu_{\psi_{Dr+\epsilon}}\Big\|^2}{\Big\|\mu_{\psi_{\epsilon}}-\mu_{\psi_{r-\epsilon}} \Big\|^2} +O(\epsilon)&\leq \frac{N_1 }{N_3}\frac{N_2+N_3}{N_1+N_2}-\frac{N_1(N_2+N_3)}{N_2N_3}\frac{\Big\|\mu_{\psi_{0}}-\mu_{\psi_{2\epsilon}}\Big\|^2}{\Big\|\mu_{\psi_{\epsilon}}-\mu_{\psi_{r-\epsilon}} \Big\|^2}\\
\iff \frac{\Big\|\mu_{\psi_{r-\epsilon}}-\mu_{\psi_{Dr+\epsilon}}\Big\|^2}{\Big\|\mu_{\psi_{\epsilon}}-\mu_{\psi_{r-\epsilon}} \Big\|^2}+O(\epsilon)&\leq \frac{N_1}{N_3}\left(\frac{N_2+N_3}{N_1+N_2}-\left(1+\frac{N_3}{N_2}\right)\frac{\Big\|\mu_{\psi_{0}}-\mu_{\psi_{2\epsilon}}\Big\|^2}{\Big\|\mu_{\psi_{\epsilon}}-\mu_{\psi_{r-\epsilon}} \Big\|^2} \right).
\end{split}
\end{equation}
\end{proof}
First note that the norms in equation (\ref{eq:proof_thm_3:k_means_fail}) are deterministic quantities that depend on $\epsilon$, $r$ and $D$. The $N_i$ are Binomial random variables parametrized by $\lambda_1$ and $\lambda_2$, i.e. $N_1\sim\text{Binom}(n,\lambda_1/2)$, $N_2\sim\text{Binom}(n,\lambda_1/2)$ and $N_3\sim\text{Binom}(n,\lambda_2)$. All terms involving $N_i's$ w.h.p. concentrate around their expectation. Thus, choosing $\lambda_1\gg\lambda_2$ allows us to make the fraction $\frac{N_1}{N_3}$ w.h.p. arbitrarily large. Choosing $\epsilon$ small enough (in comparison to $r$) ensures that the $O(\epsilon)$ term on the LHS is small enough, and that the bracketed term on the RHS is at least $\frac{1}{4}$.

\section{Sufficient conditions for Consistency of \actr, \affk, and \alnk. (Proof of Theorem 2)}

\begin{proof}[\textbf{Proof of Theorem \ref{thm:consistency_sufficient}: Consistency of \actr}]

Let $\trueMixMsr$ be any mixing measure for which there exists some $\epsilon > 0$ such that,
\begin{equation}
    \label{eq:suff_conds_ctr}
         \mathbb{P}_{\sample_n}\Big (  \frac{1}{4} \inf \limits_{k \neq k'} \rho(\trueCmpMsr_k, \trueCmpMsr_{k'}) < \sup \limits_{\substack{x \in \sample_n}} \rho(\compKDEGamma_x, \estCmpMsr_{\sigma^*(x), \sigma^*}) + \epsilon \Big ) \stackrel{n \rightarrow \infty}{\longrightarrow} 0.
\end{equation}
Then, with high probability (w.h.p) over the samples $\sample_n$,
\begin{equation}
     \inf \limits_{k \neq k'} \rho(\trueCmpMsr_k, \trueCmpMsr_{k'}) > 4 \sup \limits_{\substack{x \in \sample_n}} \rho(\compKDEGamma_x, \estCmpMsr_{\sigma^*(x), \sigma^*}) + 4 \epsilon.
\end{equation}
If the bandwidth parameter $\beta$ is chosen according to (\ref{eq:bandwidth_supp}),
\begin{equation}
    \label{eq:bandwidth_supp}
    \beta \rightarrow 0, \quad \frac{n \beta^d}{\log n} \rightarrow \infty \; \textrm{as} \; n \rightarrow \infty,
\end{equation}
it is known that the corresponding kernel density estimate $\hat{f}_n$ converges to the true density $f$ in the $l_{\infty}$ norm \citep{gine2002rates, einmahl2005uniform}. Observe that the density functions $\widehat{f}_{k, \sigma^*}$ corresponding to the planted partitions $\estCmpMsr_{k, \sigma^*}$ are the kernel density estimates of the density functions corresponding to the component distributions $\trueCmpMsr_k.$ Furthermore by assumption, we have that the corresponding component weights $\trueCmpwgt_k$ are bounded away from $0.$ Thus, for each $k \in [K]$, we have 
\begin{equation*}
    \sup \limits_{x \in \R^d} \abs{\widehat{f}_{k, \sigma^*} - f_k} \stackrel{\mathbb{P}}{\longrightarrow} 0 \; \; \text{as } n \rightarrow \infty.
\end{equation*}

An application of Scheffe's theorem (or Reiz's theorem) \citep{scheffe1947useful} implies that the corresponding probability measures $\estCmpMsr_{k, \sigma^*}$ also converge weakly to $\trueCmpMsr_k.$ \citet[Theorem 4.2]{simon2020metrizing} provide a characterization of the class of kernels that metrize the weak convergence of probability measures on locally compact domains (e.g., $\R^d$). Following \citet[Corollary 3]{simon2016kernel} and \citet[Proposition 5]{sriperumbudur2010hilbert}, one can verify that the Gaussian kernel belongs to this class of kernel functions. Therefore, weak convergence of probability measures $\estCmpMsr_{k, \sigma^*}$ to $\trueCmpMsr_k$ is equivalent to convergence in MMD with respect to (w.r.t) a Gaussian kernel, that is, for every $\epsilon > 0$,
\begin{equation}
    \label{eq:conv_prob_ctr_supp}
    \mathbb{P}(\rho(\estCmpMsr_{k, \sigma^*}, \trueCmpMsr_k) > \epsilon) \stackrel{n \rightarrow \infty}{\longrightarrow} 0.
\end{equation}
Let $t = 4\epsilon/2$ and $\delta = 1/n$. Then, for every $k \in [K]$, there exists some $N_t \in \N$ such that $\forall$ $n > N_{t,k}$,  
\begin{equation}
     \mathbb{P}(\rho(\estCmpMsr_{k, \sigma^*}, \trueCmpMsr_k) > 4\epsilon / 2) < \frac{1}{n}.
\end{equation}
Let $N_t = \sup_{k \in [K]} N_{t,k}$. For all $n > N_t$, with high probability (w.h.p) over the samples $\sample_n$, 

\begin{equation}
    \inf \limits_{k \neq k'} \rho(\trueCmpMsr_k, \trueCmpMsr_{k'}) > 4 \sup \limits_{\substack{x \in \sample_n}} \rho(\compKDEGamma_x, \estCmpMsr_{\sigma^*(x), \sigma^*}) + 2 \rho(\estCmpMsr_{k, \sigma^*}, \trueCmpMsr_k).
\end{equation}

By assumption, we have that $\trueCmpwgt_k$ is bounded away from $0$ for all $k \in [K].$ Therefore, 

\begin{equation}
    \mathbb{P}(\min \limits_{k \in [K]} \abs{(\sigma^*)^{-1}(k)} > 0 ) = \prod \limits_{k=1}^{K} \mathbb{P}(\abs{(\sigma^*)^{-1}(k)} > 0).
    \end{equation}

For any $k \in [K]$, observe that $\abs{(\sigma^*)^{-1}(k)}$ is a binomial random variable, $Bin(n, \trueCmpwgt_k).$ Using Hoeffding's inequality for binomial random variables, 
\begin{equation}
    \mathbb{P}(\abs{(\sigma^*)^{-1}(k)} \leq t) < \exp{(-2n (\trueCmpwgt_k - \frac{t}{n})^2)}
\end{equation}
Setting $t=0$, for large enough $n$ such that $n / \log n > 1 / \trueCmpwgt_k, $ w.p.a.l $1 - 1/n $ 
\begin{equation*}
    \abs{(\sigma^*)^{-1}(k)} > 0
\end{equation*}
So w.h.p over the samples, 

\begin{equation*}
    \min \limits_{k \in [K]} \abs{(\sigma^*)^{-1}(k)} > 0
\end{equation*}

From Propositions \ref{prop:recovery_kcenter_supp}, \ref{prop:suff_conds_kmeans++}, and \ref{prop:single_linkage_recovery}, we then have that w.h.p over $\sample_n$, the algorithms \actr, \affk, and \alnk can recover the planted partition $\truePart$ (upto a permutation over the labels). 
\end{proof}

\subsection{Sufficient conditions for consistency of kernel k-center clustering \actr}

\begin{proposition}[\textbf{Conditions for recovery of the true partition by kernel k-center algorithm}]
\label{prop:recovery_kcenter_supp}
For any $\trueMixMsr \in \mixingSpace$, let $\trueMsr = m(\trueMixMsr)$. Let $X = \myCurls{x_1, x_2, \cdots x_n} \sim \trueMsr^{n}$. Define $\estMsr = \sum \limits_{i=1}^n \frac{1}{n} \compKDEGamma_i$ as the probability measure associated with the kde in the usual way. For any partition $\sigma:[n] \rightarrow [K]$ such that the following condition holds:
\begin{equation}
\label{eq:suff_conds_kcenter_supp}
    \inf \limits_{ k \neq k'} \rho(\trueCmpMsr_k, \trueCmpMsr_k') > 4 \sup \limits_{i \in [n]} \rho(\psi_i, \estCmpMsr_{\sigma(i), \sigma}) + 2 \sup \limits_{k \in [K]} \rho(\estCmpMsr_{k, \sigma}, \trueCmpMsr_{k, \sigma}), 
\end{equation}
and 
\begin{equation}
\label{eq:non_empty_clusters}
   \inf \limits_{k \in [K]} \abs{\sigma^{-1}(k)} > 0
\end{equation}
$\sigma$ can be recovered by the kernel k-center algorithm on the sample kernel matrix G (defined in section \ref{sec:kde_clustering_equivalence} of the main paper). 
\end{proposition}

\begin{proof}[\textbf{Proof of Proposition \ref{prop:recovery_kcenter_supp}}]
For any sample $\nSample$ and a partition $\sigma'$, let 
\begin{equation}
    r = \sup \limits_{i \in [n]} \rho(\psi_i, \estCmpMsr_{\sigma'(i), \sigma'})
\end{equation}
We first show that for any mixing measure satisfying the conditions provided in Equation (\ref{eq:suff_conds_kcenter_supp}) w.r.t a sample $X$ and a partition $\sigma'$, then for any $i \neq j \in [n]$,
\begin{align*}
    \rho(\psi_i, \psi_j) \leq 2r &\iff \sigma'(i) = \sigma'(j) \\
    \rho(\psi_i, \psi_j) > 2r &\iff \sigma'(i) \neq \sigma'(j)
\end{align*}
\textbf{1)} $\sigma'(i) = \sigma'(j) \implies \rho(\psi_i, \psi_j) \leq 2r.$ For any $i \in [n]$, by definition,
\begin{equation}
    \rho(\psi_i, \estCmpMsr_{\sigma'(i), \sigma'}) \leq r
\end{equation}
Therefore, for any $i, j \in [n]$,
\begin{equation}
\label{eq:kcenter_leq_2r}
    \sigma'(i) = \sigma'(j) \implies \rho(\psi_i, \psi_j) \leq \rho(\psi_i, \estCmpMsr_{\sigma'(i), \sigma'}) + \rho( \estCmpMsr_{\sigma'(i), \sigma'}, \psi_j) \leq 2r
\end{equation}
\end{proof}
\textbf{2)} $\sigma'(i) \neq \sigma'(j) \implies \rho(\psi_i, \psi_j) > 2r.$ Let $\sigma'(i) = k \neq k' = \sigma'(j).$ Then, by triangle inequality, 
\begin{align}
\label{eq:kcenter_g_2r}
       \rho(\psi_i, \psi_j) & \geq \rho(\trueCmpMsr_k, \trueCmpMsr_{k'}) - \rho(\trueCmpMsr_k, \estCmpMsr_{k, \sigma'}) - \rho(\estCmpMsr_{k, \sigma'}, \psi_{i}) - \rho(\psi_j, \estCmpMsr_{k', \sigma'}) - \rho(\estCmpMsr_{k', \sigma'}, \trueCmpMsr_{k'}) > 2r
\end{align}
Combining Equations (\ref{eq:kcenter_leq_2r}) and (\ref{eq:kcenter_g_2r}), its easy to verify that
\begin{align*}
    \rho(\psi_i, \psi_j) \leq 2r &\iff \sigma'(i) = \sigma'(j) \\
    \rho(\psi_i, \psi_j) > 2r &\iff \sigma'(i) \neq \sigma'(j)
\end{align*}

For any partition $\sigma$, let
\begin{equation}
    L(\sigma) = \sup \limits_{i \in [n]} \rho(\psi_i, \estCmpMsr_{\sigma(i), \sigma}).
\end{equation}
Then the partition $\widehat{\sigma}$ generated by the kernel k-center clustering algorithm is given by
\begin{equation}
    \widehat{\sigma} = \argmin \limits_{\sigma:[n] \rightarrow [K]} L(\sigma).
\end{equation}
Then, by definition, 
\begin{equation}
\label{eq:less_than_r_supp}
    L(\widehat{\sigma}) \leq L(\sigma') = r
\end{equation}

Therefore, from (\ref{eq:less_than_r_supp}), 
\begin{equation}
\label{eq:dt_betn_diff_centers}
    \rho(\estCmpMsr_{\sigma'(i), \sigma'}, \estCmpMsr_{\widehat{\sigma}(i), \widehat{\sigma}} ) \leq \rho(\estCmpMsr_{\sigma'(i), \sigma'}, \psi_i) + \rho(\estCmpMsr_{\widehat{\sigma}(i), \widehat{\sigma}}) \leq 2r
\end{equation}

To show that the partitions $\sigma'$ and $\widehat{\sigma}$ coincide up to a permutation, we show that, for any $i, j \in [n]$, $\sigma'(i) = \sigma'(j) \implies \widehat{\sigma}(i) = \widehat{\sigma}(j)$ and $\sigma'(i) \neq \sigma'(j) \implies \widehat{\sigma}(i) \neq \widehat{\sigma}(j).$
\vspace{3mm}

Consider $i, j \in [n]$ such that $\sigma'(i) \neq \sigma'(j).$ If $\widehat{\sigma}(i) = \widehat{\sigma}(j),$ then from triangle inequality and (\ref{eq:dt_betn_diff_centers}),
\begin{equation}
    \rho(\estCmpMsr_{\sigma'(i), \sigma'}, \estCmpMsr_{\sigma'(j), \sigma'}) \leq \rho(\estCmpMsr_{\sigma'(i), \sigma'}, \estCmpMsr_{\widehat{\sigma}(i), \widehat{\sigma}}) + \rho(\estCmpMsr_{\sigma'(j), \sigma'}, \estCmpMsr_{\widehat{\sigma}(i), \widehat{\sigma}}) \leq 4r.
\end{equation}
However, from (\ref{eq:suff_conds_kcenter_supp}) we have that 
\begin{equation}
    \rho(\estCmpMsr_{\sigma'(i), \sigma'}, \estCmpMsr_{\sigma'(j), \sigma'}) \geq \rho(\trueCmpMsr_{\sigma'(i)}, \trueCmpMsr_{\sigma'(j)}) - \rho(\estCmpMsr_{\sigma'(i), \sigma'}, \trueCmpMsr_{\sigma'(i)}) - \rho(\estCmpMsr_{\sigma'(j), \sigma'}, \trueCmpMsr_{\sigma'(j)}) > 4r,
\end{equation}
which is a contradiction. Therefore, for any $i,j \in [n]$ such that 
\begin{equation}
\label{eq:same_clusters}
    \sigma'(i) \neq \sigma'(j) \implies \widehat{\sigma}(i) \neq \widehat{\sigma}(j).
\end{equation}

\vspace{3mm}
Consider any $i, j \in [n]$ such that $\sigma'(i) = \sigma'(j)$ but $\widehat{\sigma}(i) \neq \widehat{\sigma}(j).$ From (\ref{eq:dt_betn_diff_centers}) we know that 
\begin{equation}
\label{eq:two_centers_one_ball}
    \estCmpMsr_{\widehat{\sigma}(i), \widehat{\sigma}} \in B(\estCmpMsr_{\sigma'(i), \sigma'}, 2r) \; \; \textrm{and} \; \; \estCmpMsr_{\widehat{\sigma}(j), \widehat{\sigma}} \in B(\estCmpMsr_{\sigma'(i), \sigma'}, 2r)
\end{equation}
where $B(x, r) = \myCurls{y: \rho(x, y) \leq r}$ denotes the ball of radius $r$ centered at $x$. 

\vspace{3mm}
From the condition (\ref{eq:non_empty_clusters}) that the clusters are non-empty, for each $k \in [K]$, there exists $a_k$ such that $\sigma'(a_k) = k.$ Then, for each $k \in [K]$, we know that 
\begin{equation}
\label{eq:new_center_in_old_ball}
     \estCmpMsr_{\widehat{\sigma}(a_k), \widehat{\sigma}} \in B(\estCmpMsr_{\sigma'(a_k), \sigma'}, 2r) = B(\estCmpMsr_{k, \sigma'}, 2r)
\end{equation}

Furthermore, observe that for all $k \neq k' \in [K],$
\begin{equation}
\label{eq:balls_empty_intersection}
    B(\estCmpMsr_{k, \sigma'}, 2r) \cap B(\estCmpMsr_{k', \sigma'}, 2r) = \emptyset,
\end{equation}

since otherwise there exists some $x \in B(\estCmpMsr_{k, \sigma'}, 2r) \cap B(\estCmpMsr_{k', \sigma'}, 2r) $, i.e., 

\begin{align*}
     & \rho(x, \estCmpMsr_{k, \sigma'}) \leq 2r \; \; \textrm{and} \; \; \rho(x, \estCmpMsr_{k', \sigma'})  \leq 2r, \\
     \implies & \rho(\estCmpMsr_{k, \sigma'}, \estCmpMsr_{k', \sigma'}) \leq \rho(x, \estCmpMsr_{k, \sigma'}) + \rho(x, \estCmpMsr_{k', \sigma'}) \leq 4r,
\end{align*}
which is a contradiction. 

\vspace{3mm}

Moreover, by definition, $\sigma'(a_k) \neq \sigma'(a_{k'})$ for all $k, k' \in [K]$, from (\ref{eq:same_clusters}), we have
\begin{equation}
\label{eq:all_unequal_new_centers}
    \widehat{\sigma}(a_1) \neq \widehat{\sigma}(a_{2}) \cdots \neq \widehat{\sigma}(a_K)
\end{equation}

Since there are only $K$ centers, (\ref{eq:new_center_in_old_ball}),  (\ref{eq:balls_empty_intersection}) and (\ref{eq:all_unequal_new_centers}) imply that 
\begin{itemize}
    \item For any $i \in [n]$, there exists some $k \in [K]$ such that $\widehat{\sigma}(i) = \widehat{\sigma}(a_k),$ and
    \item $\estCmpMsr_{\widehat{\sigma}(a_k), \widehat{\sigma}} \in B(\estCmpMsr_{\sigma'(i), \sigma'}, 2r) \implies \estCmpMsr_{\widehat{\sigma}(a_{k'}), \widehat{\sigma}} \notin B(\estCmpMsr_{\sigma'(i), \sigma'}, 2r)$ for all $k' \neq k \in [K].$
\end{itemize}
So, from (\ref{eq:two_centers_one_ball}),
\begin{equation}
   \sigma'(i) = \sigma'(j) \implies \estCmpMsr_{\widehat{\sigma}(i), \widehat{\sigma}} = \estCmpMsr_{\widehat{\sigma}(j), \widehat{\sigma}} \implies \widehat{\sigma}(i) = \widehat{\sigma}(j),
\end{equation} 
since, if $\widehat{\sigma}(i) \neq \widehat{\sigma}(j)$, then $\rho(\estCmpMsr_{\widehat{\sigma}(i), \widehat{\sigma}}, \estCmpMsr_{\widehat{\sigma}(j), \widehat{\sigma}}) > 4r.$

\vspace{3mm}
Therefore, the partitions $\sigma'$ and $\widehat{\sigma}$ coincide up to a permutation over the labels.

\subsection{Sufficient conditions for kernel kmeans++ algorithm - proofs}


\begin{proposition}[\textbf{Sufficient conditions for recovery by kernel k-means ++}]
\label{prop:suff_conds_kmeans++}
For any $\trueMixMsr \in \mixingSpace$, let $\trueMsr = m(\trueMixMsr)$. Let $X = \myCurls{x_1, x_2, \cdots x_n} \sim \trueMsr^{n}$. Define $\estMsr = \sum \limits_{i=1}^n \frac{1}{n} \compKDEGamma_i$ as the probability measure associated with the kde in the usual way. For any partition $\sigma':[n] \rightarrow [K]$ such that the following condition holds:
\begin{equation}
\label{eq:suff_conds_kmeans++_supp}
    \inf \limits_{ k \neq k'} \rho(\trueCmpMsr_k, \trueCmpMsr_k') > 4 \sup \limits_{i \in [n]} \rho(\psi_i, \estCmpMsr_{\sigma'(i), \sigma'}) + 2 \sup \limits_{k \in [K]} \rho(\estCmpMsr_{k, \sigma'}, \trueCmpMsr_{k, \sigma'}), 
\end{equation}
and 
\begin{equation}
\label{eq:non_empty_clusters_supp}
   \inf \limits_{k \in [K]} \abs{(\sigma')^{-1}(k)} > 0
\end{equation}
$\sigma$ can be recovered by a (deterministic) kernel k-means++ algorithm on the sample kernel matrix G.

\end{proposition}

\begin{proof}[\textbf{Proof of Proposition \ref{prop:suff_conds_kmeans++}}]
Let,
\begin{equation}
    r = \sup \limits_{i \in [n]} \rho(\psi_i, \estCmpMsr_{\sigma'(i), \sigma'}), \textrm{ and } B_k = B(\estCmpMsr_{k, \sigma'}, r) \; \; \forall k \in [K].
\end{equation}

\begin{claim}
Let $C$ be the set of centers initialized in phase one of the k-means ++ algorithm as described. Then, for each $k \in [K],$ 
\begin{equation}
    c_k \in B_k
\end{equation}
\end{claim}

\begin{claimproof}
For every $i \in [n],$ by definition, 
\begin{equation}
    \rho(\psi_i, \estCmpMsr_{\sigma'(i), \sigma'}) \leq r \implies \psi_i \in B_{\sigma'(i)} .
\end{equation}

Therefore, without loss of generality (W.L.O.G), let $c_1 \in  B_1.$ For any $t < K$, assume that $C_t = \myCurls{c_1, c_2, \cdots c_t} $ and $c_k \in B_k \; \; \forall k \in [t]$ (upto a permutation over the labels). Note that $B_k$ is non-empty for every $k \in [K].$ 

From the proof of Proposition \ref{prop:recovery_kcenter_supp}, for any mixing measure satisfying the conditions provided in (\ref{eq:suff_conds_kmeans++_supp}), 
\begin{align}
\label{eq:dt_<2r_>2r}
    \rho(\psi_i, \psi_j) \leq 2r &\iff \sigma'(i) = \sigma'(j) \\
    \rho(\psi_i, \psi_j) > 2r &\iff \sigma'(i) \neq \sigma'(j)
\end{align}

Therefore, since $c_k \in B_k$ for all $k \in [K]$, $d(\psi_i) =  \rho^2(\psi_i, c_k) \leq 2r \; \textrm{ for all } \sigma'(i) = k.$ Therefore,
\begin{equation}
    d(\psi_i) \textrm{ is } \begin{cases}
    \leq 2r & \forall \psi_i \in B_k \textrm{, and } k \leq t, \\
    > 2r & \textrm{otherwise.}
    \end{cases}
\end{equation}

Since $c_{t+1} = \argmax \limits_{\psi_i} d(\psi_i),$ $c_{t+1} \in B_{s}$ for some $s \notin C_t.$ 

\end{claimproof}

\begin{claim}
Kernel k-means algorithm does not affect the centers obtained in Phase one of the algorithm.
\end{claim}

\begin{claimproof}
From claim 1, in phase one of the algorithm, the centers $C = \myCurls{c_1, c_2, \cdots c_K}$ are obtained such that $c_k \in B_k$ for all $k \in [K].$ For each $k \in [K]$, clusters $\myCurls{C_1, C_2, \cdots C_K}$ are then defined as follows.
\begin{equation}
    C_k = \myCurls{i \in [n]: \rho^2(c_k, \psi_i) \geq \rho^2(c_{k'}, \psi_i) \quad \forall k \neq k' \in [K]}
\end{equation}
From (\ref{eq:dt_<2r_>2r}), we have that 
\begin{align*}
    \rho^2(\psi_i, c_k) \leq 4r^2 \quad & \textrm{ if } \sigma'(i) = k \\
    \rho^2(\psi_i, c_k) > 4r^2 \quad & \textrm{ otherwise .}
\end{align*}
 
 Therefore, the partition obtained in the Phase 1 of the algorithm coincides with $\sigma'$ up to a permutation over the labels, that is,
 \begin{equation}
     C_k = \myCurls{\psi_i \in \sample: \sigma'(i) = k},
 \end{equation}
and
\begin{equation}
    \sum \limits_{i: \sigma'(i) = k} \psi_i = \estCmpMsr_{k, \sigma'} \in B_k.
\end{equation}
Clearly, 
\begin{equation*}
    \rho(\compKDEGamma_i, \estCmpMsr_{\sigma'(i), \sigma'}) \leq 2r \leq \rho(\compKDEGamma_i, \estCmpMsr_{k, \sigma'}) > 2r \; \forall k \neq \sigma'(i).
\end{equation*}
Therefore, the clusters obtained in the phase 1 of the algorithm do not change in the Phase 2 of the algorithm and the partition obtained by \affk coincides with that of $\sigma'$ up to a permutation over the labels.


\end{claimproof}
\end{proof}

\subsection{Sufficient conditions for kernel linkage clustering algorithms (Proof of Theorem 2 - Part III)}
\begin{proposition}[\textbf{Recovery by single linkage clustering}]
\label{prop:single_linkage_recovery}
For any $\trueMixMsr \in \mixingSpace$, let $\trueMsr = m(\trueMixMsr)$. Let $X_n = \myCurls{x_1, x_2, \cdots x_n} \sim \trueMsr^{n}$ be a sample. Define $\estMsr = \sum \limits_{i=1}^n \frac{1}{n} \compKDEGamma_i$ as the probability measure associated with the kde in the usual way. For any partition $\sigma_n$ such that the following condition holds:
\begin{equation}
\label{eq:suff_conds_sl}
    \inf \limits_{ k \neq k'} \rho(\trueCmpMsr_k, \trueCmpMsr_k') > 3 \sup \limits_{k} \sup \limits_{l \neq l' \in \sigma_n^{-1}(k)} \rho( \compKDEGamma_l, \compKDEGamma_{l'}) + 2 \sup \limits_{k \in [K]} \rho(\estCmpMsr_{k, \sigma_n}, \trueCmpMsr_{k, \sigma_n}),
\end{equation}
$\sigma_n$ can be recovered by the kernel single (and complete) linkage clustering algorithms with respect to the Gaussian kernel with bandwidth para using the sample kernel matrix G (defined in section \ref{sec:kde_clustering_equivalence} of the main paper). 
\end{proposition}
\begin{proof}[\textbf{Proof of proposition \ref{prop:single_linkage_recovery}.}]
For any partition $\sigma$, let 
$$\delta = \sup \limits_{k \in [K]} \sup \limits_{i, j' \in \sigma^{-1}(k)} \rho(\compKDEGamma_i, \compKDEGamma_j).$$
We first show that for any partition $\sigma$ satisfying the conditions stated in Proposition \ref{prop:single_linkage_recovery},
\begin{align*}
    \forall l, l' \in [n] \qquad &\sigma(l) = \sigma(l')  \iff \rho(\compKDEGamma_l, \compKDEGamma_{l'}) \leq \delta,  \\
     &\sigma(l) \neq \sigma(l')  \iff \rho(\compKDEGamma_l, \compKDEGamma_{l'}) > \delta.
\end{align*}

Observe that, by definition, 
\begin{equation}
\label{eq:bydef_sl}
    \forall l \neq l' \in [n], \quad \sigma(l) = \sigma(l') \implies \rho(\compKDEGamma_l, \compKDEGamma_{l'}) \leq \delta.
\end{equation}
By subadditivity of $\rho$, for any $l, l' \in [n]$ such that $\sigma(l) = k$, $\sigma(l') = k'$, and $k \neq k'$,  
\begin{align}
\label{eq:subadd_sl}
    \rho(\trueCmpMsr_k, \trueCmpMsr_{k'}) < \rho(\trueCmpMsr_k, \estCmpMsr_k) + \rho(\estCmpMsr_k, \compKDEGamma_l) + \rho(\compKDEGamma_l, \compKDEGamma_{l'}) + \rho(\compKDEGamma_{l'}, \estCmpMsr_{k'}) + \rho(\estCmpMsr_{k'}, \trueCmpMsr_{k'}).
\end{align}
Substituting (\ref{eq:suff_conds_sl}) in (\ref{eq:subadd_sl}), we obtain
\begin{equation}
\label{eq:bysubadd_sl}
    \sigma(l) \neq \sigma(l')  \implies \rho(\compKDEGamma_l, \compKDEGamma_{l'}) > \delta.
\end{equation}

Using the fact that $\rho(\cdot, \cdot) \geq 0$, from (\ref{eq:bydef_sl}) and (\ref{eq:bysubadd_sl}), we have

\begin{align*}
    \forall l, l' \in [n] \qquad &\sigma(l) = \sigma(l')  \iff \rho^2(\compKDEGamma_l, \compKDEGamma_{l'}) \leq \delta^2, \\
     &\sigma(l) \neq \sigma(l')  \iff \rho^2(\compKDEGamma_l, \compKDEGamma_{l'}) > \delta^2 .
\end{align*}

All three linkage algorithms based on the matrix of squared MMD evaluations between the component distributions $\myCurls{\compKDEGamma_l}_{l=1}^{n}$ or alternatively using the sample kernel matrix $G$ (see Lemma \ref{lemma:mmd_kernel_eqv}) would first group the components within the same cluster according to $\sigma$ before grouping components belonging to different clusters according to $\sigma$. Therefore, thresholding the dendrogram to obtain exactly $K$ clusters would recover the underlying partition $\sigma$ upto a permutation over the labels. With a minor modification of the proof, it is easy to see that the Proposition also holds under separbility conditions provided in (\ref{eq:suff_conds_kmeans++_supp}).

\end{proof}


\begin{proof}[\textbf{Proof of Theorem \ref{thm:identifiability_lnk}: Consistent recovery of the planted partition by $\mathcal{A}_{\textrm{LNK}}$}]

Let $\trueMixMsr$ be any mixing measure for which there exists some $\epsilon > 0$ such that,
\begin{equation}
    \label{eq:suff_conds_sl_supp}
        \mathbb{P}_{\sample_n}\left( \sup \limits_{\substack{x, x' \in \sample_n:\\ \truePart(x) = \truePart(x')}} \rho(\compKDEGamma_x, \compKDEGamma_{x'})   > \frac{1}{3} \inf \limits_{k \neq k'} \rho(\trueCmpMsr_k, \trueCmpMsr_{k'}) - \epsilon \right) \stackrel{n \rightarrow \infty}{\longrightarrow} 0,
\end{equation}
Then, with high probability (w.h.p) over the samples $\sample_n$,
\begin{equation}
     \inf \limits_{k \neq k'} \rho(\trueCmpMsr_k, \trueCmpMsr_{k'}) > 3 \sup \limits_{\substack{x, x' \in \sample_n:\\ \truePart(x) = \truePart(x')}} \rho(\compKDEGamma_x, \compKDEGamma_{x'}) + 3 \epsilon.
\end{equation}
Furthermore, we know that for every $\epsilon > 0$,
\begin{equation}
    \label{eq:conv_prob_supp}
    \mathbb{P}(\rho(\estCmpMsr_{k, \sigma^*}, \trueCmpMsr_k) > \epsilon) \stackrel{n \rightarrow \infty}{\longrightarrow} 0.
\end{equation}
Let $t = 3\epsilon/2$ and $\delta = 1/n$. Then, for every $k \in [K]$, there exists some $N_t \in \N$ such that $\forall$ $n > N_{t,k}$,  
\begin{equation}
     \mathbb{P}(\rho(\estCmpMsr_{k, \sigma^*}, \trueCmpMsr_k) > 3\epsilon / 2) < \frac{1}{n}.
\end{equation}
Let $N_t = \sup_{k \in [K]} N_{t,k}$. For all $n > N_t$, with high probability (w.h.p) over the samples $\sample_n$, 

\begin{equation}
    \inf \limits_{k \neq k'} \rho(\trueCmpMsr_k, \trueCmpMsr_{k'}) > 3 \sup \limits_{\substack{x, x' \in \sample_n:\\ \truePart(x) = \truePart(x')}} \rho(\compKDEGamma_x, \compKDEGamma_{x'}) + 2 \rho(\estCmpMsr_{k, \sigma^*}, \trueCmpMsr_k).
\end{equation}

From Proposition \ref{prop:single_linkage_recovery}, we have that w.h.p over $\sample_n$, kernel single linkage clustering algorithm recovers the true partition $\truePart$ (upto a permutation over the labels).

\end{proof}

\section{Necessary conditions for consistency of \affk \textrm{ }and \alnk. (Proof of Theorem 3)}
\subsection{Proof for $\mathcal{A}_{\text{FFK}}$}
Fix the kernel bandwidth parameter $\zeta>0$. Let $r$, $\epsilon$ and $K$ be small constants that satisfy $1>r>2K>16\epsilon$. Consider the following example in $\mathbb{R}$, where  $\mathcal{U}([a,b])$ denotes the uniform distribution on the real interval $[a,b]$. Let
\begin{equation}
\gamma_1 =m\left(\frac{1}{2}\mathcal{U}([-\epsilon,\epsilon])+\frac{1}{2}\mathcal{U}([r-\epsilon,r+\epsilon])\right)
\end{equation}
and
\begin{equation}
\gamma_2 =m\left(\frac{1}{2}\mathcal{U}([2r-K-\epsilon,2r-K+\epsilon])+\frac{1}{2}\mathcal{U}([3r-K-\epsilon,3r-K+\epsilon])\right).
\end{equation}The mixing measure is given by $\Lambda=\frac{1}{2}\gamma_1+\frac{1}{2}\gamma_2$. The idea is that because $K>0$, the two clusters are just not separated enough.  

To see that $\mathcal{A}_{\text{FFK}}$ fails to recover the planted partition with probability approaching $\frac{1}{2}$, consider the case where the first cluster center is initialized with a point $c_1\in [r-\epsilon,r+\epsilon]$. The farthest first heuristic then chooses a second cluster center $c_2\in [3r-K\epsilon,3r-K+\epsilon]$. Since $K>4\epsilon$, the initial clusters will be given by 
\begin{equation*}
C_1=\{x:x\leq 2r-K+\epsilon\}\quad\text{and}\quad C_2=\{x:x\geq 3r-K-\epsilon\}.
\end{equation*}Consequently, in the first iteration of phase two of the algorithm (compare section \ref{sec:A_FFK}), the new cluster centers satisfy
\begin{equation*}
\tilde c_1\geq \frac{rN_2+(2r-K)N_3}{N_1+N_2+N_3}-\epsilon\quad\text{and}\quad\tilde c_2\geq 3r-K-\epsilon,
\end{equation*}where $N_i$ denotes the number of points within the respective intervals. Now the clusters themselves do not change if
\begin{equation*}
\begin{split}
&(2r-K)+\epsilon-\tilde c_1 \leq \tilde c_2 - (2r-K) -\epsilon\\[6pt]
\iff & \frac{2N_1+N_2}{N_1+N_2+N_3}r-\frac{N_1+N_2}{N_1+N_2+N_3}K\leq r-4\epsilon,
\end{split}
\end{equation*}an event that occurs asymptotically almost surely as the $N_i$ concentrate around their expectation. Conditional on this event, the algorithm terminates with clusters $C_1$ and $C_2$, i.e. it does not recover the planted partition. Due to symmetry, the same holds if the first cluster center is initialized with a point in $ [2r-K-\epsilon,2r-K+\epsilon]$. As $n\to\infty$, the probability to initialize the first cluster center with a point in either $[r-\epsilon,r+\epsilon]$ or $ [2r-K-\epsilon,2r-K+\epsilon]$ approaches $\frac{1}{2}$.

We now show that the condition in the theorem is satisfied, namely that as $n\to\infty$, it holds that
\begin{equation}\label{eq:proof_thm2_statement}
    \frac{\rho(\trueCmpMsr_1, \trueCmpMsr_{2})}{\sup \limits_{\substack{x \in \sample_n}} \rho(\compKDEGamma_x, \estCmpMsr_{\sigma^*(x), \sigma^*})}>4-\hat\epsilon.
\end{equation}
A simple way to evaluate the LHS is to express both numerator and denominator as sums of inner products between Gaussians. We have
\begin{equation*}
\rho(\trueCmpMsr_1,\trueCmpMsr_2)\geq \rho(\estCmpMsr_{1, \sigma^*}, \estCmpMsr_{2, \sigma^*})-\rho(\trueCmpMsr_1,\estCmpMsr_{1, \sigma^*})-\rho(\trueCmpMsr_2,\estCmpMsr_{2, \sigma^*}),
\end{equation*}and as $n\to\infty$ and $\beta\to 0$, the latter two terms converge in probability to 0. Hence, for all $\epsilon_1>0$, it holds that
\begin{equation*}
\rho^2(\trueCmpMsr_1,\trueCmpMsr_2)\geq \rho^2(\estCmpMsr_{1, \sigma^*}, \estCmpMsr_{2, \sigma^*})-\epsilon_1.
\end{equation*}Furthermore, since $\rho^2$ is bounded, for all $n$ large enough
\begin{equation*}
\rho^2(\trueCmpMsr_1,\trueCmpMsr_2)\geq \mathbb{E}\left[\rho^2(\estCmpMsr_{1, \sigma^*}, \estCmpMsr_{2, \sigma^*})\right]-2\epsilon_1.
\end{equation*}
A straightforward if somewhat lengthy calculation shows that
\begin{equation}
\mathbb{E}\left[\rho^2(\hat\gamma_{1, \sigma^\star}, \hat\gamma_{2, \sigma^\star})\right]\geq\frac{2}{\zeta}(2r-K)^2+O(\epsilon)+o(r^4).
\end{equation}Similarly, for the denominator,
\begin{equation}
\sup \limits_{\substack{x \in \sample_n}} \rho^2(\compKDEGamma_x, \estCmpMsr_{\sigma^*(x), \sigma^*})\leq \frac{2}{\zeta}\frac{1}{4}r^2+O(\epsilon).
\end{equation}Hence,
\begin{equation*}
\begin{split}
\frac{\rho^2(\hat\gamma_{1, \sigma^\star}, \hat\gamma_{2, \sigma^\star})}{\sup \limits_{\substack{x \in \sample_n}} \rho^2(\compKDEGamma_x, \estCmpMsr_{\sigma^*(x), \sigma^*})}&\geq \frac{(2r-K)^2+O(\epsilon)+o(r^4)-2\epsilon_1}{\frac{1}{4}r^2+O(\epsilon)}\\
&\geq\frac{16-2\frac{K}{r}+O\left(\frac{\epsilon}{r^2}\right)+o(r^2)+\frac{2\epsilon_1}{r^2}}{1+O\left(\frac{\epsilon}{r^2}\right)}.
\end{split}
\end{equation*}Thus, in order to satisfy \eqref{eq:proof_thm2_statement}, we have to choose $r$  small enough, and $K$, $\epsilon$ and  $\epsilon_1$ small enough in comparison to $r$. We now derive the expression for the numerator. First define the sets $I_1=\{x\in X_n:x\in[-\epsilon,\epsilon]\}$, $I_2=\{x\in X_n:x\in[r-\epsilon,r+\epsilon]\}$, $I_3=\{x\in X_n:x\in[2r-K-\epsilon,2r-K+\epsilon]\}$ and $I_4=\{x\in X_n:x\in[3r-K-\epsilon,3r-K+\epsilon]\}$. Denote $N_i=|I_i|$. We have
\begin{equation*}
\begin{split}
    \rho^2(\hat\gamma_{1, \sigma^*_n}, \hat\gamma_{2, \sigma^*_n})=&<\hat\gamma_{1, \sigma^*_n}, \hat\gamma_{1, \sigma^*_n}>+<\hat\gamma_{2, \sigma^*_n}, \hat\gamma_{2, \sigma^*_n}>-2<\hat\gamma_{1, \sigma^*_n}, \hat\gamma_{2, \sigma^*_n}>\\[6pt]
    =&\frac{\sum_{x,y\in  I_1}<\psi_x,\psi_y>+2\sum_{x\in  I_1,y\in I_2}<\psi_x,\psi_y>+\sum_{x,y\in  I_2}<\psi_x,\psi_y>}{(N_1+N_2)^2}\\
    &+\frac{\sum_{x,y\in  I_3}<\psi_x,\psi_y>+2\sum_{x\in  I_3,y\in I_4}<\psi_x,\psi_y>+\sum_{x,y\in  I_4}<\psi_x,\psi_y>}{(N_3+N_4)^2}\\
    &-2\frac{\sum_{x\in  I_1,y\in I_3}<\psi_x,\psi_y>+\sum_{x\in  I_1,y\in I_4}+\sum_{x\in  I_2,y\in I_3}<\psi_x,\psi_y>+\sum_{x\in  I_2,y\in I_4}<\psi_x,\psi_y>}{(N_1+N_2)(N_3+N_4)}\\
    \geq&\sqrt{\frac{\zeta}{\eta}}\Bigg[\frac{N_1^2(1-\frac{4\epsilon^2}{\eta})+2N_1 N_2(1-\frac{(r+2\epsilon)^2}{\eta})+N_2^2(1-\frac{4\epsilon^2}{\eta})}{(N_1+N_2)^2}\\
    &+\frac{N_3^2(1-\frac{4\epsilon^2}{\eta})+2N_3 N_4(1-\frac{(r+2\epsilon)^2}{\eta})+N_4^2(1-\frac{4\epsilon^2}{\eta})}{(N_3+N_4)^2}\\
    &-2\frac{N_1N_3(1-\frac{(2r-K-2\epsilon)^2}{\eta})+N_1N_4(1-\frac{(3r-K-2\epsilon)^2}{\eta})}{(N_1+N_2)(N_3+N_4)}\\
    &-2\frac{N_2N_3(1-\frac{(r-K-2\epsilon)^2}{\eta})+N_2N_4(1-\frac{(2r-K-2\epsilon)^2}{\eta})}{(N_1+N_2)(N_3+N_4)}\Bigg]+o(r^4)
    \end{split}
\end{equation*}
Where we used \eqref{eqn:inner_kme} and the Taylor expansion $e^x=1+x+o(x^2)$. The inequality sign stems from the fact that we have replaced the exact locations of sampled points with interval boundaries. Taking expectations, 
\begin{equation*}
\begin{split}
\mathbb{E}\left[\rho^2(\hat\gamma_{1, \sigma^*_n}, \hat\gamma_{2, \sigma^*_n})\right]\geq&\sqrt{\frac{\zeta}{\eta}}\frac{1}{\eta}\Bigg[\frac{-4\epsilon^2-2(r+2\epsilon)^2-4\epsilon^2}{4}+
    \frac{-4\epsilon^2-2(r+2\epsilon)^2)-\epsilon^2}{4}\\
    &+2\frac{(2r-K-2\epsilon)^2+(3r-K-2\epsilon)^2}{4}
    +2\frac{(r-K-2\epsilon)^2+(2r-K-2\epsilon)^2}{4}\Bigg]+o(r^4)\\
    =&\frac{2}{\eta}\sqrt{\frac{\zeta}{\eta}}(2r-K)^2+O(\epsilon)+o(r^4).
\end{split}
\end{equation*}

We now derive the expression for the denominator. By symmetry, it suffices to consider the case $x\in[-\epsilon,\epsilon]$.
\begin{equation*}
\begin{split}
    &\rho\left(\compKDEGamma_x, \frac{1}{N_1+N_2}\left(\sum_{x'\in[-\epsilon,\epsilon]} \compKDEGamma_{x'} + \sum_{x'\in[r-\epsilon,r+\epsilon]}\compKDEGamma_{x'}
    \right)\right)\\
    =&\frac{1}{N_1+N_2}||\sum_{x'\in[-\epsilon,\epsilon]}(\compKDEGamma_{x'}-\compKDEGamma_{x}) + \sum_{x'\in[r-\epsilon,r+\epsilon]}(\compKDEGamma_{x'}-\compKDEGamma_{x}) ||\\
    \leq&\frac{N_1}{N_1+N_2}\rho(\compKDEGamma_{-\epsilon},\compKDEGamma_{\epsilon})+\frac{N_2}{N_1+N_2}\rho(\compKDEGamma_{-\epsilon},\compKDEGamma_{r+\epsilon})\\
    \leq& \rho(\compKDEGamma_{-\epsilon},\compKDEGamma_{+\epsilon})+\frac{N_2}{N_1+N_2}\rho(\compKDEGamma_{0},\compKDEGamma_{r})\\
    =&\sqrt{2\sqrt{\frac{\zeta}{\eta}}\left(1-e^{-\frac{4\epsilon^2}{\eta}}\right)}+\frac{N_2}{N_1+N_2}\sqrt{2\sqrt{\frac{\zeta}{\eta}}\left(1-e^{-\frac{r^2}{\eta}}\right)}\\
    \leq&\frac{N_2}{N_1+N_2}r\sqrt{\frac{2}{\eta}}\sqrt[4]{\frac{\zeta}{\eta}}+O(\epsilon)
\end{split}
\end{equation*}
where we used \eqref{eq:mmd_for_gaussians} and the inequality $1-e^{-x}\leq x$. It follows that asymptotically almost surely
\begin{equation*}
\sup \limits_{\substack{x \in \sample_n}} \rho^2(\compKDEGamma_x, \estCmpMsr_{\sigma^*(x), \sigma^*})\leq \frac{2}{\eta}\sqrt{\frac{\zeta}{\eta}}\frac{1}{4}r^2+O(\epsilon).
\end{equation*}

\subsection{Proof for $\mathcal{A}_{\text{LNK}}$}
Consider the same example as in the above proof for $\mathcal{A}_{\text{FFK}}$. At first, a hierarchical linkage algorithm (compare section \ref{sec:A_LNK}) will merge all points within $2\epsilon$-intervals. This leaves us with 4 trees. Then, the linkage algorithm does \textit{not} return the planted partition if the trees belonging to the intervals $[r-\epsilon,r+\epsilon]$ and $[2r-K\epsilon,2r-K+\epsilon]$ are merged in the next step. For $r\gg K\gg\epsilon$, it can be easily seen that this is the case.

\section{Statistical identifiability with respect to \ectr, \effk, and \elnk}

\begin{proof} [\textbf{Proof of Theorem 5: Consistency implies statistical identifiability}]
Let $\trueMixMsr$ be 

For appropriate choice of bandwidths, we know that 
\begin{equation}
\label{eq:regularity_conv_supp}
    \lim \limits_{n \rightarrow \infty} \rho(\estCmpMsr_{k , \sigma^*_n}, \trueCmpMsr_k) \stackrel{\mathbb{P}}{=} 0 \qquad \textrm{and} \qquad \lim \limits_{n \rightarrow \infty} \abs{\estCmpwgt_{k, \sigma^*_n} - \trueCmpwgt_k} \stackrel{\mathbb{P}}{=} 0.
\end{equation}

From \citet[Lemma A.3]{aragam2018identifiability}, convergence of component measures and the corresponding component weights implies that the sequence of estimators defined by $\estMixMsr = \sum \limits_{i=1}^{K} \estCmpwgt_{k, \sigma^*_n} \delta_{\estCmpMsr_{k, \sigma^*_n}}$ converges in probability to the true mixing measure $\trueMixMsr$ w.r.t the Wasserstein metric.
\end{proof}



\section{Estimating the Bayes partition}
Given a finite sample $\nSample$, let $\widehat{\sigma}$ denote the partition generated by a kernel clustering algorithm $\mathcal{A}.$ We can define an estimator of the Bayes partition function $\widehat{\sigma}_b:\R^d \rightarrow [K]$ in the natural way:

\begin{equation}
\label{eqn:reassignment_step_supp}
    \widehat{\sigma}_b(x) = \argsup \limits_{k \in [K]} \sum \limits_{j: \widehat{\sigma}(j) = k} G_{\beta}(x, x_j) \stackrel{(*)}{=} \argsup \limits_{k \in [K]} \estCmpwgt_{k, \widehat{\sigma}} \widehat{f}_{k, \widehat{\sigma}}(x)
\end{equation}
where $(*)$ follows from Lemma \ref{lemma:mmd_kernel_eqv}. Due to the equivalence between kernel clustering and density-based clustering, we can show that if a kernel-based algorithm $\mathcal{A}$ can consistently recover the planted partition, then by means of a single reassignment step given by (\ref{eqn:reassignment_step_supp}), the algorithm  consistently recovers the Bayes partition. 

\textbf{Exceptional set.} Given $\trueMixMsr = \sum_{k \in [K]} \trueCmpwgt_k \delta_{\trueCmpMsr_k}$, for any $t > 0$, we define the exceptional set
\begin{equation*}
\label{eq:exceptional_set_supp}
    E(t) = \bigcup \limits_{k \neq k'} \myCurls{x \in \R^d : \abs{\trueCmpwgt_k f_k(x) - \trueCmpwgt_{k'} f_{k'}(x)} \leq t}.
\end{equation*}
\begin{theorem}[\textbf{Estimating the Bayes partition}]
\label{thm:Bayes_consistency_supp}
Let $\zeta$, and $\beta$ be bandwidth parameters satisfying the conditions provided in Theorem \ref{thm:consistency_sufficient}. Let $\trueMixMsr \in \prob^2_{K}$ satisfying the conditions provided in (\ref{eq:suff_conds_ctr_supp}). For $\nSample \sim m(\trueMixMsr)^{n}$ and let $\widehat{\sigma}_{b, n}$ be the partition function obtained by \actr, \affk \textrm{ or } \alnk \textrm{ } followed by the reassignment step in (\ref{eqn:reassignment_step_supp}). Then, w.h.p over the samples, there exists a sequence $\myCurls{t_n} \stackrel{n \rightarrow \infty}{\longrightarrow} 0$ such that $\widehat{\sigma}_n(x) = \BayesPart(x)$ for all $x \in \R^d - E_0(t_n).$
\end{theorem}


\begin{proof}[\textbf{Proof of Theorem \ref{thm:Bayes_consistency}}]
The proof of this Proposition is adapted with minor changes from the proof of \citet[Theorem 5.2]{aragam2018identifiability}. For this reason, we borrow some of the notation from \citet{aragam2018identifiability}. Since $\trueMixMsr$ satisfies the separability conditions given in equation (\ref{eq:suff_conds_sl_supp}), from Theorem \ref{thm:consistency_sufficient}, we know that w.h.p over the samples the algorithms \actr, \affk, and \alnk recover the planted partition  up to a permutation over the labels, that is, $\widehat{\sigma} = \sigma^*$. For appropriate choice of bandwidths, we know that w.h.p over the samples,
\begin{equation}
\label{eq:densities_unif_conv}
    \lim \limits_{n \rightarrow \infty} f_{k, \sigma^*} \stackrel{\mathbb{P}}{=} f_k,
\end{equation}
where the convergence is defined pointwise and uniformly over $\R^d$.


Let, 
\begin{equation}
\label{eq:def_tn}
    t_n = 2 \sup \limits_{k \in [K]} \sup \limits_{x \in \R^d} \abs{\estCmpwgt_{k, \sigma^*_n} \widehat{f}_{k, \sigma^*_n}(x) - \trueCmpwgt_k f_k(x)} \geq 0.
\end{equation}
From (\ref{eq:densities_unif_conv}), we know that $t_n \stackrel{\mathbb{P}}{\longrightarrow} 0.$ Moreover, by definition, we have that 
\begin{equation}
\label{eq:diff_fk}
    \abs{\trueCmpwgt_k f_k(x) - \trueCmpwgt_{k'} f_{k'}(x)} > t_n \implies \trueCmpwgt_{\sigma_{Bayes}(x)} f_{\sigma_{Bayes}(x)}(x) > \trueCmpwgt_{k} f_{k}(x) + t_n \; \forall x \not \in E_0(t_n), \; k \neq \sigma_{Bayes}(x).
\end{equation}
Therefore, it follows that for any $x \in \R^D -  E_0(t_n)$ and any $k \neq \sigma_{Bayes}(x)$,
\begin{equation}
  \estCmpwgt_{\sigma_{Bayes}(x), \sigma^*_n}  \widehat{f}_{\sigma_{Bayes}(x), \sigma^*_n}(x)  \stackrel{(1)}{>}  \trueCmpwgt_{\sigma_{Bayes}(x)} f_{\sigma_{Bayes}(x)} - \frac{t_n}{2} \stackrel{(2)}{>} \trueCmpwgt_{k} f_k(x) + \frac{t_n}{2} \stackrel{(3)}{>} \estCmpwgt_{k, \sigma^*_n}  \widehat{f}_{k, \sigma^*_n}(x),
\end{equation}
where, (1) and (3) follow from (\ref{eq:def_tn}) and (2) follows from (\ref{eq:diff_fk}). This implies that $\widehat{\sigma}_b(x) = \argsup \limits_{k \in [K]} \estCmpwgt_{k, \sigma^*} \estCmpdensity_{k, \sigma^*} (x) =  \sigma_{Bayes}(x)$ for all $ x \not \in E_0(t_n).$
\end{proof}

    

\end{document}